\documentclass[twoside]{article}

\usepackage[accepted]{aistats2021}
\usepackage{custom}

\begin{document}

\runningtitle{MBVI for SDEs}


\twocolumn[

\aistatstitle{Moment-Based Variational Inference \\ for Stochastic Differential Equations}

\aistatsauthor{Christian Wildner \And Heinz Koeppl}

\aistatsaddress{Technische Universit\"at Darmstadt \\ \href{mailto:christian.wildner@bcs.tu-darmstadt.de}{\nolinkurl{christian.wildner@bcs.tu-darmstadt.de}}  \And  Technische Universit\"at Darmstadt \\  \href{mailto:heinz.koeppl@bcs.tu-darmstadt.de}{\nolinkurl{heinz.koeppl@bcs.tu-darmstadt.de}}  } ]

\begin{abstract}
Existing deterministic variational inference approaches for diffusion processes use simple proposals and target the marginal density of the posterior. We construct the variational process as a controlled version of the prior process and approximate the posterior by a set of moment functions. In combination with moment closure, the smoothing problem is reduced to a deterministic optimal control problem. Exploiting the path-wise Fisher information, we propose an optimization procedure that corresponds to a natural gradient descent in the variational parameters. Our approach allows for richer variational approximations that extend to state-dependent diffusion terms. The classical Gaussian process approximation is recovered as a special case.
\end{abstract}

\section{INTRODUCTION}

It\^{o} processes governed by a stochastic differential equation (SDE) are an important class of time series models involving uncertainty. Originating from the statistical physics of diffusion, SDEs have become an important modeling tool in areas as diverse as biology, finance and engineering. However, applying SDEs as a predictive tool requires learning model parameters from real data. Usually, such data is corrupted by noise and only available at discrete sampling times. In such a scenario, likelihood-based parameter inference requires estimation of the posterior over the latent process. Computing this posterior requires the solution of a PDE that is only computationally tractable for very low-dimensional state spaces or for linear systems (see \citet{sarkka_2019} for an accessible introduction). 
Thus, standard approximations linearize the system dynamics or use a discrete time approximation. In a Bayesian setting, Monte Carlo methods such as MCMC, SMC or particle MCMC methods are a common \citep{golightly_2011}. In practice, sampling-based methods often struggle with high dimensional settings or with highly informative observations \citep{delmoral_2014}. In such a scenario, variational inference \citep{blei_2017} may provide a more scalable alternative.

\paragraph{Related Work}
The variational formulation of Bayesian inference of latent stochastic processes and its connection to stochastic control have been observed first by \citet{mitter_2003}. \citet{archambeau_2007a} introduced variational inference for SDEs to the machine learning community. Their core idea is to compute the best linear Gaussian process approximation of the posterior. While this approach has been refined and extended several times over the years \citep[e. g.][]{vrettas_2011, ruttor_2013, duncker_2019}, it is limited to state independent diffusion terms. An alternative approach presented by \citet{sutter_2016} constructs the variational process such that the marginal density belongs to a prespecified exponential family. While overcoming the Gaussian limitation, the construction is also mathematically involved. \citet{cseke_2016} suggested an approximation of the posterior in terms of moments rather than the marginal density within an expectation propagation framework for smoothing. Another moment-based approximation, albeit in the context of Markov jump processes, was proposed by \citet{wildner_2019}. However, the key idea of transition space partitioning for complexity reduction cannot be applied to SDEs. The main drawback of the deterministic approaches above is that they rely on model-specific derivations. Sampling-based variational inference does not require such computations and can also be applied to SDEs \citep{ryder_2018}. However, this comes at the price of much longer training times.

\paragraph{Contributions}
 In this work, we propose a new sampling-free structured variational approach to latent diffusion processes that mitigates some drawbacks of earlier methods. Similarly to the approach of \citet{cseke_2016}, we construct the proposal process as a controlled version of the prior process and reduce complexity by projecting the stochastic process onto a collection of summary statistics. To solve the variational problem, we adapt a strategy proposed by \citet{wildner_2019}. Using the Markov property in combination with moment closure, we map the full smoothing problem to a deterministic optimal control problem. Exploiting the path-wise Fisher information, we construct an effective natural gradient descent in the variational parameters. To keep model-specific derivations at a minimum, we implement our method in the PyTorch framework. Thus, we can circumvent a large part of the model-specific computations by exploiting Pytorch's automatic differentiation capabilities. Exploiting the structural similarity to the moment-based approach to Markov jump processes, we provide a unified framework capable of handling both SDEs and MJPs. The accompanying code is available at \url{https://git.rwth-aachen.de/bcs/projects/cw/public/mbvi_sde}.

\section{PRELIMINARIES} \label{sec:background}

This section summarizes material on SDEs, the inference problem for noisy observations in discrete time and the general variational formulation.

\subsection{Stochastic Differential Equations} \label{sec:sde}

Let $\mathcal X \subset \mathbb R^n$. We consider a stochastic processes $X$ on $\mathbb R^n$ over a finite time interval $[0,T]$ given by the It\^{o} SDE
\begin{align} \label{eq:sde}
\mathrm d X_t  = a( X_t) \mathrm dt + b(X_t) \mathrm d W_t \, . 
\end{align}
Here, $W$ is an $n$-dimensional Wiener process and $a: \mathbb R^n \rightarrow \mathbb R^n$, $b: \mathbb R^{n} \rightarrow \mathbb R^{n \times n}$ are functions of suitable regularity, i.e. satisfying a Lipschitz condition.  Additionally, we will focus on cases where $b(x)$ has full rank for all $x \in \mathcal X$. The solution of an SDE of the form \eqref{eq:sde} is a Markov process and the corresponding marginal density satisfies the Fokker-Planck equation. 
%The marginal density $p: \mathbb R^n \times [0,T] \rightarrow \mathbb R_+$ satisfies the Fokker-Planck equation
%\begin{align} \label{eq:fokker_planck}
%\partial_t p(x, t) = [A p](x,t)
%\end{align}
%where $A$ is the (forward) generator given by
%\begin{equation} \label{eq:forward_generator} 
%\begin{split}
%[Af] (x) &=  - \sum_{i=1}^n \partial_i [ a_i(x, t) f(x) ] +\frac{1}{2} \sum_{i, j = 1}^N  \partial_i \partial_j [  D_{ij} (x, t) f(x) ] 
%\end{split} \, . 
%\end{equation}
%Here  $D(x,t):= b (x, t) b(x,t)^T$ denotes the diffusion tensor. 
In practice, one is often not interested in the full density but rather certain summary statistics $S: \mathbb R^n  \rightarrow \mathbb R^l$. Often, $S$ will correspond to first and second order monomials but other choices are possible as well. Now define the moment functions $\varphi(t)  := \mathsf E[ S( X_t) ]$. The idea is now to propagate $\varphi$ in time rather than the density. One can show that the moment functions $\varphi_i$ satisfy a system of differential equations 
\begin{align} \label{eq:moment_equation}
\dot  \varphi_i(t)  = \mathsf E[ A^\dagger S_i (X_t) ] \, 
\end{align}
where the backward generator $A^\dagger$ is the $L_2$-adjoint of the Fokker-Planck operator and given by
\begin{align} \label{eq:backward_generator}
[A^\dagger f ](x) = \sum_{i=1}^n  a_i(x) \partial_i f( x) + \frac{1}{2} \sum_{i, j = 1}^N D_{ij} (x) \partial_i \partial_j f(x)
\end{align}
for $f \in \mathcal C^2(\mathbb R^n )$ \citep{ethier_2005}. The diffusion tensor $D$ is determined by the SDE \eqref{eq:sde} through the relation $D= b b^T$. In general, the system \eqref{eq:moment_equation} is not closed in $\varphi$, i.e. it will be of the form
\begin{align} \label{eq:moment_dynamics}
\dot \varphi(t) =  B \varphi(t) + B' \mathsf E[ S' (X_t ) ] \, .
\end{align}
Here, $B$ and $B'$ are matrices of suitable dimension and $S'$ corresponds to a collection of higher order moments. Thus, Eq. \eqref{eq:moment_dynamics} still depends on the full process $X$. In order to obtain a closed form description, one can employ moment closure \citep{kuehn_2016}. A general closure is given by a function $h$ that approximates the higher order moments $S'$ such that \eqref{eq:moment_dynamics} reduces to
\begin{align} \label{eq:closed_dynamics}
\dot \varphi(t) =  B \varphi(t) + B' h( \varphi (t) ) \, .
\end{align}
Two common methods to obtain closure schemes are via extensions and truncation of the summary statistics and by assuming an underlying distribution. In this work, we focus on the latter approach as it has been shown to correspond to a projection of the stochastic process onto a parametric family of distributions \citep{bronstein_2018}.

\subsection{Posterior Path Estimation} \label{sec:conditional_processes}

We consider a scenario where the underlying process $X$ is not observed directly. Instead, we have access to sparse and noisy observations $Y = (Y_1, \ldots, Y_n)^T$ obtained at sample times $0 \leq t_1 \leq \ldots \leq t_n \leq T$. We assume that the observations are conditionally independent given the latent path of $X$ and follow a noise distribution $Y_i \sim P_\mathrm{obs} ( \cdot \mid X(t_i ) ) $. 
The smoothing problem refers to evaluating expectations of the form $\mathsf E[ f( X_t) \mid \sigma(Y)]$ where $\sigma(Y)$ denotes the history of the observation process $Y$ up to the terminal time $T$. Under mild conditions, $\mathsf E[ f( X_t) \mid \sigma(Y)]$ can be represented by a conditional probability density $\pi (x, t \mid  y_1, \ldots, y_n)$. Now $\pi$ can be understood as the marginal density of a posterior process $\bar X$. The posterior process $\bar X$ obeys an SDE with the same diffusion term as the prior process \eqref{eq:sde} and a modified drift
\begin{align} \label{eq:posterior_drift}
\bar a(x,t) = a(x) + D(x) \nabla \log ( \beta(x, t) ) \,  
\end{align}
where the source term $\beta$ satisfies a backward equation \citep{archambeau_2011}
\begin{align*}
\beta( x, t)  = - A^\dagger \beta (x, t)  \, .
\end{align*}
 Intuitively, \eqref{eq:posterior_drift} corresponds to a controlled version of the prior process where the second term steers the process towards future observations. This analogy is the main motivation underlying our structured variational approximation introduced in Sec. \ref{sec:variational_class}.

\subsection{Variational Smoothing}

Let $\mu$ and $\nu$ be probability measures on a common probability space such that $\mu$ is absolutely continuous with respect to $\nu$. Recall that the Kullback-Leibler divergence or relative entropy between $\mu$ and $\nu$ is defined as 
\begin{equation*}
D_\mathrm{KL} [ \, \mu \, || \, \nu ] = \int \log \left( \frac{ d \mu }{ d \nu } \right) \mathrm d \mu \, . 
\end{equation*}
Now consider two diffusions $Z$, $X$ with drifts $a^{Z}$, $a^X$ respectively and a shared diffusion tensor $D$ that is invertible for almost every $x \in \mathcal X$. Then the Kullback-Leibler divergence on the level of sample paths is given by
\begin{equation} \label{eq:path_kl}
\begin{split}
&D_\mathrm{KL} [ P^Z \, || \, P^X ] = \int_0^T \mathsf E \Bigl[  \left( a^Z( Z_t) - a^X(Z_t) \right)^T  \\
  &\quad \times   D(Z_t)^{-1} \left( a^Z( Z_t) - a^X(Z_t) \right) \Bigr] \mathrm dt \, ,
\end{split}
\end{equation}
where $P^Z$, $P^X$ denote the measures over sample paths induced by the processes $Z$ and $X$, respectively. 
A rigorous exposition on the relative entropy of diffusion processes is given in \citet{mitter_2003}. More intuitively,  the path divergence \eqref{eq:path_kl} can be derived by considering the divergence of a corresponding discrete time system and taking the continuum limit \citep{ archambeau_2007a, archambeau_2007}. For variational smoothing, we aim to find an approximate process $Z$ within a class  $\mathfrak Z$ of simpler processes. Following the usual variational inference framework \citep{blei_2017}, the best approximation $Z^*$ within $\mathfrak Z$  is given by 
\begin{align*}
Z^* =  \arg \min_{Z \in \mathfrak Z} D_\mathrm{KL}[ P^Z \, || \, P^{\bar X} ]  \, .
\end{align*}
By inserting the true posterior drift \eqref{eq:posterior_drift}, one can show that this objective function decomposes into 
\begin{equation} \label{eq:objective_function}
\begin{split}
&D_\mathrm{KL}[ P^Z \, || \, P^{\bar X} ]  =  D_\mathrm{KL}[ P^Z \, || \, P^{ X} ]  \\
&\quad - \sum_{k=1}^n \mathsf E[ \log p( y_k \mid Z_{t_k} ) ] + \log C
\end{split}
\end{equation}
where $X$ is the prior process and $C = \mathsf E[ p( y_1, \ldots, y_n \mid X(t_1), \ldots, X(t_n ) ) ]$ is the evidence.

\section{VARIATIONAL SMOOTHING} \label{sec:variational_smoothing}

\subsection{Structured Variational Approximation} \label{sec:variational_class}

 From the posterior drift \eqref{eq:posterior_drift}, we observe that the true posterior process $\bar X$ is a controlled version of the prior process $X$. The idea is now to approximate the driving term in $\eqref{eq:posterior_drift}$ by a feedback control. This leads to a drift of the form
\begin{align} \label{eq:stat_control}
a^Z(z, t) = a^X(z) + R(x) v(t) T(x) \, .
\end{align}
Here $v: [0,T] \rightarrow \mathbb R^{n \times m} $ is a deterministic, matrix-valued function corresponding to the variational parameters while $T:\mathbb R^n \times [0,T] \rightarrow \mathbb R^m$ represents a collection of control features and $R : \mathbb R^n \rightarrow \mathbb R^{n\times n} $ is a rescaling matrix. Typically, $R$ will be set he identity, the diffusion term $b$ or the diffusion tensor $D$. Suitable choices of the rescaling factor can simplify the resulting equations and also reduce the computational complexity of the algorithm. A more detailed discussion is given in App. \ref{ssec:rescaling}.
In general, the control features $T$ will be different from the summary statistics $S$. In the simple case where $T$ is the identity map, \eqref{eq:stat_control} corresponds to a linear feedback control. 
For the following discussion, we also  introduce $u:[0,T] \rightarrow \mathbb R^{n m}$ as a vectorized control obtained by stacking the columns of $v$. 
\begin{lemma} \label{thm:quadratic_form}
Under the variational drift \eqref{eq:stat_control}, the KL-term in the objective function \eqref{eq:objective_function} becomes a quadratic form in the vectorized controls $u$ and can be represented as
\begin{equation} \label{eq:quadratic_form}
D_\mathrm{KL} [ P^Z \, || \, P^{X} ] =  \frac{1}{2} \int_0^T   u(t)\top g(\varphi(t)) u(t) \,  \mathrm dt \, , 
\end{equation}
where the matrix valued function $g:\mathbb R^l \rightarrow R^{nm \times nm}$ is determined by the diffusion tensor $D$, the rescaling matrix R, the control features $T$ and the summary statistics $S$.  
\end{lemma} 
\begin{proof}[Proof sketch]
First, we show by direct calculation that under the variational drift \eqref{eq:stat_control} the KL term can be written as 
\begin{equation*}
   D_\mathrm{KL}[ P^Z \, || \, P^{X} ] =  \frac{1}{2} \int_0^T   u(t)^T \mathsf E[ \psi( x(t)) ]u(t) \,  \mathrm dt 
\end{equation*}
with $\psi: \mathbb R^n \rightarrow \mathbb R^{nm \times nm}$ such that
\begin{align*}
   \psi(x) = \begin{pmatrix}
   T_1 T_1 \tilde D^{-1} & \ldots & T_1 T_m   \tilde D^{-1} \\
   \vdots & \cdots & \vdots \\
   T_mT_1 \tilde D^{-1} & \ldots, & T_m  T_m  \tilde D^{-1}
   \end{pmatrix} (x)
\end{align*}
with $\tilde D^{-1} = R^T D^{-1} R$. Under a suitable choice of the summary statistics $S$, one can express the expectation  as $ \mathsf E[ \psi(Z_t) ]  = g( \varphi(t) ) $. Such a $g$ can always be found, e.g. by augmenting the summary statistics $S$ accordingly. The details are given in App. \ref{ssec:proof_lemma_1}. 
\end{proof} 
 The full variational inference problem now corresponds to minimizing the objective function $\eqref{eq:objective_function}$ with respect to $u$ and $\varphi$ subject to the moment equation \eqref{eq:moment_equation}. Since \eqref{eq:moment_equation} still depends on the full stochastic process $Z$, we consider instead a relaxed variational inference problem by replacing the exact moment constraint with
\begin{equation} \label{eq:forward_equation}
\dot \varphi(t) = f(u(t), \varphi(t) ) \, . 
\end{equation}
where $f$ is obtained from a closure scheme. The relaxation of the moment constraint simplifies the variational inference problem considerably as summarized in the following proposition. 
\begin{prop} \label{thm:control_problem}
The relaxed variational inference problem corresponds to a finite dimensional deterministic optimal control problem of the form
\begin{equation}  \label{eq:control_problem}
\begin{aligned}
\min_{u, \varphi} & &  & J[u, \varphi] \\
\mathrm{s.t. } & &  &\dot  \varphi(t) = f( u(t), \varphi(t) )    \,  \\   
\end{aligned} \, 
\end{equation}
with
\begin{align} \label{eq:full_objective}
   J[u, \varphi] = \int_0^T L( u(t), \varphi(t) )\, \mathrm dt -\sum_{k=1}^n F_k(\varphi(t_k) )
\end{align}
where 
\begin{align*}
F_k(\varphi(t_k) )) = \mathsf E[ \log p(y_k \mid Z_{t_k})]
\end{align*}
represents the contributions of the observations in \eqref{eq:objective_function} expressed in terms of $\varphi$. The cost function $L$ is given by
\begin{align} \label{eq:moment_kl_psi}
 L( u(t), \varphi(t) ) = \frac{1}{2} u(t)^T g ( \varphi(t) )  u(t)    \, . 
\end{align}
\end{prop}
Proposition \ref{thm:control_problem} is a consequence of Lemma \ref{thm:quadratic_form} in combination with the moment closure relaxation. A detailed discussion is given in App. \ref{ssec:proof_prop_1}. 

%While moment closure approximations may seem like a hack, it has been shown recently that a well-designed closure scheme corresponds to a projection of the stochastic process onto a parametric family of distributions \cite{bronstein_2018}. However, we note that 

%\paragraph{Rescaling}
%Suppose the diffusion term $b$ of the prior model (c.f. Eq. \eqref{eq:sde}) is available, invertible and of reasonable form. Setting $R(x,t):= b(x,t)$, the KL-term in the objective function 
%becomes 
%\begin{equation} \label{eq:diff_stat_control_kl}
%\begin{split}
%&D [ P^Z \, || \, P^{X} ] =  \frac{1}{2} \int_0^T \mathrm{Tr} \left( \mathsf E \left[ T (Z_t)  T(Z_t)^T  \right]   u(t)^T   u(t)  \right)   \, .
%\end{split}
%\end{equation}
%This follows from a short direct calculation calculation (see Supplement). Note that    

\subsection{Gradient-Based Optimization} \label{sec:optimization}

 A standard approach to solve control problems of the form \eqref{eq:control_problem} is a gradient descent in the controls $u$ \citep[see e.g.][]{stengel_1994}. 
While such a gradient descent may work in principle, it often suffers from slow convergence. We can do better in our scenario by exploiting the probabilistic nature of the objective function. The key insight here is that the variational family induces a statistical manifold on the sample path space parametrized by the controls. This allows us in a first step to construct the path-wise Fisher information which we then use to derive a natural gradient descent \citep{amari_1998} in the controls $u$. 
\begin{lemma} \label{thm:fisher}
Let $Z$ and $Z'$ be two members of the variational process family parametrized by $u$ and $u'$ respectively. We then have
\begin{align} \label{eq:kl_second_order}
D_\mathrm{KL} [ P^Z \, || \,  P^{Z'} ] = \frac{1}{2} G(u)[ u- u', u- u' ]
\end{align}
where $G(u)[ \cdot, \cdot ]$ for fixed $u$ is a symmetric positive semidefinite bilinear form given by
\begin{align*}
&G(u)[ u'-u, u'-u] = \int_0^T (u'(t) - u(t) )^T \\
&\quad \times g( \varphi (t) )  (u'(t) - u(t) ) \,  \mathrm d t \, . 
\end{align*}
\end{lemma}
Lemma \ref{thm:fisher} can be proved very similarly to Lemma \ref{thm:quadratic_form}. For completeness, the proof is provided in App. \ref{ssec:proof_lemma_2}. Now the Fisher information corresponds to the second order approximation of $D_\mathrm{KL} [ P^Z \, || \,  P^{Z'} ]$ as $u'$ approaches $u$. Since the divergence is already a quadratic form, it follows immediately from Lemma \ref{thm:fisher} that $G(u)[ \cdot, \cdot]$ is the path-wise Fisher information at $u$. This allows us to construct natural gradient updates to solve the control problem \eqref{eq:control_problem}. Both optimization algorithms featured in this work are summarized in the following proposition. An algorithmic representation is given in Alg. \ref{alg:ngd}.
\begin{prop} \label{thm:natural_gradient}
The regular (RGD) and natural (NGD) gradient descent updates of the control problem \eqref{eq:control_problem} with respect to the statistical manifold induced by $u$ and step size $h$ are given by
\begin{align}
\begin{split}
&u^{(i+1)}(t)   = u^{(i)} (t) \\
&\quad -h \left( g(\varphi^{(i)}(t)) u^{(i)}(t) -    {f_{u }^{(i)}}^T (t) \cdot \eta^{(i)}(t)  \right) \, ,
\end{split} \label{eq:regular_gradient_step}  \\ 
\begin{split}
&u^{(i+1)}(t)   = u^{(i)} (t) \\
&\quad -h  \left(  u^{(i)}(t) -   g(\varphi^{(i)}(t))^{-1} {f_{u }^{(i)}}^T (t) \cdot \eta^{(i)}(t)  \right) \, ,
\end{split} \label{eq:natural_gradient_step}
\end{align}
where $\varphi^{(i)}$ is the solution of the forward equation
\begin{align*}
\dot \varphi^{(i)} (t) = f(u^{(i)} (t),  \varphi^{(i)} (t) )  
\end{align*}
and $\eta^{(i)} $ is the solution  of the adjoint equation 
\begin{align} \label{eq:adjoint_equation}
\dot \eta^{(i)}(t) =   L_{\varphi}^{(i)} (t) - f_{\varphi }^{(i)} ( t )^T \cdot \eta^{(i)}(t) \, .
\end{align}
The notation $(\cdot )^{(i)}_{\varphi} $ and $(\cdot )^{(i)}_{u} $  denote the Jacobians with respect to $\varphi^{(i)}$ and $u^{(i)}$, respectively. 
\end{prop}
\begin{proof}[Proof sketch]
Since the control $u$ fully defines the moments $\varphi$, we can understand the control problem $\eqref{eq:control_problem}$ as the minimization of a functional $J[u]$. Steepest descent with respect to a local metric $G$ corresponds solving the constrained optimization problem 
\begin{equation*}
\begin{aligned}
&u^{(i+1)}   = \arg  \min_u  J[u] \\
&\mathrm{s.t. }  \quad  \frac{1}{2} G(u^{(i)} )[ u - u^{(i)} , u - u^{(i)} ]= \epsilon   \,  
\end{aligned} \, 
\end{equation*}
for small $\epsilon$ and then taking the limit $\epsilon \rightarrow 0$. For small $\epsilon$, one can expand $J[u] $ around $u^{(i)}$. Keeping only the first order term leads to a quadratic problem that can be solved with variational calculus. For RGD, we use that gradient descent corresponds to a steepest descent w.r.t. to the Euclidean metric. The result follows therefore by an identical computation with $G(u)$ replaced by the $L_2$ inner product. For the details, we refer to App. \ref{ssec:proof_prop_2}. 
\end{proof}
If the dynamic equation \eqref{eq:forward_equation} is obtained via moment closure, the summary statistics $\varphi$ will not correspond to a globally valid stochastic process. Thus, the gradient has to be understood
as an approximation as well. It is then advisable to check the results empirically by creating samples from the variational process with optimized control $u^*$.

\subsection{A Note on Implementation} \label{sec:general_purpose}

Solving the backward equation and computing the gradient updates requires the derivation of a number of model-specific functions. To reduce this overhead, we exploit the automatic differentiation capabilities of PyTorch which allows to effectively compute gradients and Jacobian-vector products. The most general version of our implementation only requires the specification of two functions: the r.h.s. of the forward equation $f$ and either $g$ or $L$. For certain subclasses, the implementation can be further simplified. In particular, we construct a general purpose method by fixing the control features and summary statistics as
\begin{equation} \label{eq:standard_method}
\begin{split}
T(x) &= (1, x_1, \ldots, x_n)^T \, ,\\
S(x) &=  (x_1, \ldots, x_n, x_1^2, x_1 x_2 , x_2^2, \ldots,  x_n^2 )^T  \, .
\end{split}
\end{equation}
Intuitively, the choice of features $T$ corresponds to a linear feedback control. The summary statistics $S$ consist of first and second order moments and thus directly correspond to the mean and covariance of the approximate posterior, which is in line with many approximate non-linear filtering techniques \cite{sarkka_2013}. 
Here, it is also convenient to represent the control in terms of functions $u_0$, $u_1$ such that we can write
\begin{align}
v(t) T(x) = u_0(t) + u_1(t) x \, . 
\end{align}
With $m(t) \equiv \mathsf E[Z_t] $ and $M(t) \equiv \mathsf E[(Z_t-m(t)) (Z_t-m(t))^T]$, we obtain from \eqref{eq:moment_equation}
\begin{equation*}
\begin{split}
&\dot m(t) = \mathsf E[ a(Z_t) ] + u_0(t) + u_1(t) m(t) ]  \\ 
&\dot M(t) = E[ a(Z_t) Z_t^T] + E[ Z_t a(Z_t)^T] + \mathsf E[ D(Z_t) ]  \\
&\quad + u_1(t) M(t) + M(t) u_1 (t)^T  \\
&\quad - \mathsf E[ a(Z_t) ]  m(t)^T - m(t) \mathsf E[ a(Z_t) ]^T 
\end{split}
\end{equation*} 
%For a clearer representation, we write the corresponding variational drift as 
%\begin{align} \label{eq:linear_control|}
%a^Z(x,t) = a(x) + v(t) T(x)  = a(x) + u_0(t) + u_1 (t) x \, . 
%\end{align}
%To write down the moment equations \eqref{eq:moment_equation},  $S$ and consider $m(t) := \mathsf E[ Z_t]$ and $M(t) := \mathsf E[ (Z_t - m(t) ) ( Z_t-m(t) )^T]$. From the general moment equation \eqref{eq:moment_equation} we get the joint system
%\begin{equation} \label{eq:mean_cov_equation}
%\begin{split}
%\dot m(t) &= \mathsf E[ a( Z_t) ] + u_0 (t) + u_1(t) m(t) \, , \\
%\dot M(t) & = \mathsf E[ a(Z_t) (Z_t-m(t) )^T ] + \mathsf E[ (Z_t- m(t) ) a(Z_t)^T ] \\
%&\quad  + \mathsf E[ D(Z_t) ] + u_1^T + u_1(t) M(t) + M(t) u_1(t)^T 
%\end{split}
%\end{equation}
%where we have suppressed the explicit time dependence of $a$ and $b$. In pratice, the coupled system \eqref{eq:mean_cov_equation} is implemented as a single ODE with only the unique components of the covariance matrix. 
Under the choice \eqref{eq:standard_method}, $f$ and $g$ can be constructed automatically by specifying $\mathsf E[ a( Z_t) ] $, $\mathsf E[ a(Z_t) Z_t^T]$ and  $\mathsf E[ D(Z_t) ]$ in terms of the first and second order moments.  This will typically require a moment closure. We include two standard closure schemes that lead to a reduction to moments of first and second order: a Gaussian closure for processes defined on the whole $\mathbb R^n$ and a log-normal closure for processes defined on $\mathbb R^n_+$ (see App. \ref{ssec:moment_closure}).
We conclude this section by commenting on the relation to the standard Gaussian process approximation \cite{archambeau_2007a, archambeau_2007}. As shown in App. \ref{ssec:vgpa}, by a suitable choice of the control features the GP approximation arises as a special case within our framework. 

\begin{algorithm}
    \caption{Robust Natural Gradient Descent for Moment-Based Variational Smoothing}
    \label{alg:ngd}
 \begin{algorithmic}[1]
    \STATE {\bfseries Input:} Initial guess $u^{(0)}$, initial condition $\varphi(0)$, learning rate $h$, step size modifiers $\alpha, \beta$.
    \FOR{i = 0, \ldots, \text{maxiter}} 
    \STATE Given $u^{(i)}$, $\varphi(0)$, compute $\varphi^{(i)}$ using \eqref{eq:forward_equation}.
    \STATE Given $u^{(i)}$, $\varphi^{(i)}$, compute $\eta^{(i)}$ using \eqref{eq:adjoint_equation}.
    \STATE Set $u'$ according to \eqref{eq:natural_gradient_step}.
    \IF{ $J[ u' ] < J[ u^{(i)} ] $ }
    \STATE  $h \rightarrow \alpha \cdot h$, $u^{(i+1)} \rightarrow u'$
    \ELSE 
    \STATE $h \rightarrow \beta \cdot h$, $u^{(i+1)} \rightarrow u^{(i+1)} $
    \ENDIF
    \ENDFOR
    \STATE {\bfseries Output:} Variational control $u^*$ .
 \end{algorithmic}
 \end{algorithm}

 \subsection{Online Variational Smoothing}

 The optimization based on Alg. \ref{alg:ngd} processes the full sequence of observations at once. This can be problematic for some dynamical systems as the initial estimate might be far away from the observations or when the variance of the prior process is very large. For such cases, we employ an online version of the variational smoother. For this online version, Alg. \ref{alg:ngd} is run for a number of steps on the first observation only. Then, the second observation is included and the smoother is initialized with the last control of the previous step. This procedure is repeated until all observations are processed.

\section{PARAMETER INFERENCE} \label{sec:inference}

Variational smoothing algorithms can be straightforwardly extended to inference of model parameters. Let $\theta$ be a collection of real-valued parameters and extend the prior model such that the drift and diffusion terms are understood as functions of $\theta$. More explicitly, replace $a(x) \rightarrow a(x, \theta)$ and $b(x) \rightarrow b(x, \theta)$ in the model given by \eqref{eq:sde}. We can now proceed along the line of Sec. \ref{sec:variational_smoothing} to derive a relaxed variational inference problem (see App. \ref{ssec:extended_control_problem}). Again the result can be phrased as a control problem
\begin{equation}  \label{eq:full_control_problem}
\begin{aligned}
\min_{\theta, u, \varphi} & &  & \int_0^T L( \theta,  u(t), \varphi(t) )\, \mathrm dt -\sum_{k=1}^n F_i(\varphi(t_k) ) \\
\mathrm{s.t. } & &  &\dot  \varphi(t) = f(\theta, u(t), \varphi(t) )    \,  \\   
\end{aligned} \, 
\end{equation}
Solving the control problem \eqref{eq:full_control_problem} is equivalent to maximizing an approximate evidence lower bound. We discuss three ways to solve \eqref{eq:full_control_problem}. In the first approach, $\theta$ and $u$ are optimized interchangeably corresponding to the usual variational expectation maximization framework.  The second idea is to construct a joint gradient descent in the parameters $\theta$ and the controls $u$. In practice, we observed that a combination of both approaches works well, where we alternately take a number of gradient steps for $\theta$ and $u$. \\
Finally, we consider a scenario were we have several independent time series samples $Y^1, \ldots, Y^N$ from the same underlying model. The standard variational inference procedure in this case requires computing $u^*_n(t)$ for each time series $Y^n$ to perform a single parameter update. This becomes intractable for larger data sets. We therefore consider an amortized approach based on an inference network. The idea is to model the controls as a parametric function of the observations. In our case, we set $u_n(t) = h( y_n, \phi )$  where $h$ is a feed-forward neural network parametrized by $\phi$. As shown in App. \ref{ssec:amortized_inference}, the corresponding optimization problem becomes 
\begin{equation}  \label{eq:autoencoder_control_problem}
\begin{aligned}
\min_{\theta, \phi, \varphi} & &  & \sum_{i=1}^N \int_0^T L( \theta,  h(y_i, \phi ), \varphi_i(t) )\, \mathrm dt -\sum_{k=1}^n F_i(\varphi_i(t_k) ) \\
\mathrm{s.t. } & &  &\dot  \varphi_i(t) = f(\theta, h(y_n), \varphi_i(t) )  \quad i = 1, \ldots, N   \,  \\   
\end{aligned} \, 
\end{equation}
For an implementation in PyTorch, we can exploit that our approach is gradient-based. Prop. \ref{thm:natural_gradient} allows us to compute the gradient of the objective function with respect to an arbitrary control $u(t)$.  We can thus backpropagate through the variational smoothing code such that it supports automatic differentiation. Conceptually, this is similar to neural ODE framework \citep{chen_2018} which allows to backpropagate through an ODE solver. Using the resulting module as the loss function, the inference network can be trained end-to-end using standard optimizers based on back-propagation. For a simple conceptual demonstration of the inference network, we refer to Sec. \ref{sec:ou_process}.

\section{EXPERIMENTS}  \label{sec:experiments}

In this section, we present four examples chosen to illustrate the versatility of our approach. For more details regarding the model equations and implementation, we refer to App. \ref{ssec:experiment_details}. 

\subsection{Regular Gradient vs. Natural Gradient}

We are interested in comparing the performance of the natural and regular gradient descent. We investigate this using the non-linear diffusion given by
\begin{equation} 
\mathrm d X_t = 4 X_t ( 1 - X_t^2) \mathrm dt + \sigma \mathrm d W_t
\end{equation}
that was also featured in the original work on Gaussian process approximations \cite{archambeau_2007a, archambeau_2007}. The drift of the system has two stable stationary points at $x = \pm 1$.  On occasion, the process noise may drive the system from one stationary point to another. We pick one fixed trajectory for which such a switch occurs. We then generate 10 different initial controls at random. For each of these initial controls, we perform the optimization with regular gradient descent and with natural gradient descent. The averaged log-transformed objective functions over gradient iterations are shown in  Fig. \ref{img:natural_gradient}. We observe that the natural gradient descent is more effective then the regular gradient descent, in particular in the middle part of the optimization. Also note that for small to medium dimensions, the computation time per gradient step is approximately equal for both methods. This is because the Fisher information is required for both (see Prop. \ref{thm:natural_gradient}).  Only for larger system, the matrix inversion in \eqref{eq:natural_gradient_step} may become prohibitive compared to the forward and backward ODE solution.

\begin{figure}[t]
\centering
\includegraphics[width=0.95\linewidth]{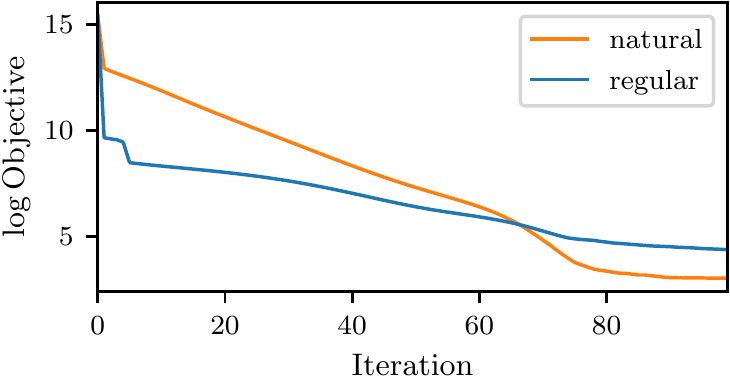}
\caption{Evolution of the objective function under natural gradient descent (red) and regular gradient descent (blue). The lines correspond to the log of the objective function averaged over 10 runs started with randomly initialized controls. }
\label{img:natural_gradient}
\end{figure}

\subsection{Joint Smoothing and Inference}

Geometric Brownian motion is a simple example of a process with a state dependent diffusion term and thus cannot be treated in the linear gaussian process framework. Here, we consider a simple multivariate extension given by the SDE system
\begin{equation} 
\mathrm d X_{i, t} = r_i X_{i, t} \mathrm d t + X_{i, t}  \mathrm d \tilde W_{i, t}
\end{equation}
\begin{figure}[t!]
\includegraphics[width=0.95\linewidth]{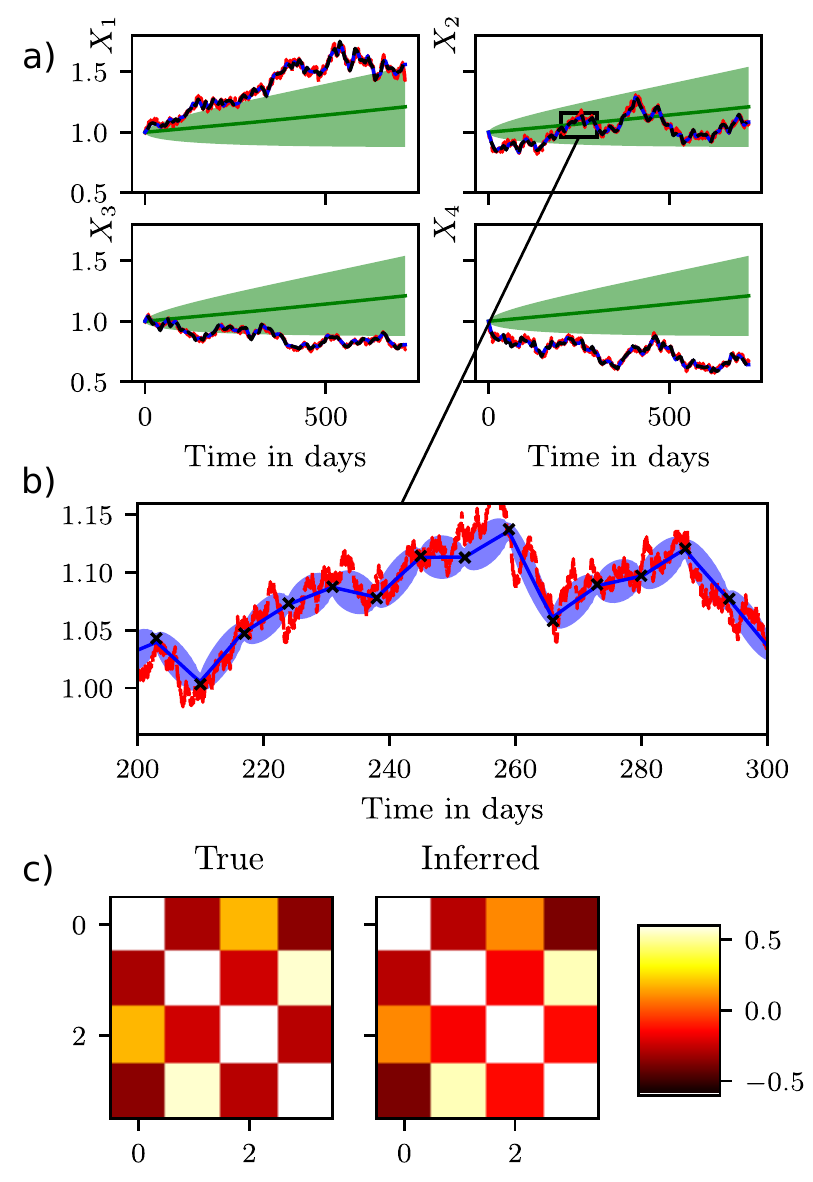}
\caption{Joint smoothing and inference for a multivariate geometric Brownian with $n=4$. a) Noisy observations, ground truth and smoothed state  for the process components. For comparison, we show the prior process initialized with uninformative parameters. The shaded region indicates the standard deviation of the prior. b) A  zoom-in showing the posterior compared to the noisy observations and the ground truth. The shaded blue region indicates the standard deviation of the variational posterior. c) Ground truth of the correlation matrix compared to reconstructed correlation matrix. }
\label{img:gbm}
\end{figure}
for each component $i$. Here $\tilde W_{i, t}$ is a collection of correlated Brownian motions. Similar as for a multivariate normal distribution, a correlated Brownian motion can be constructed as $\tilde W_t  = R W_t$ where $W_t$ is a vector of independent standard Brownian motions and the matrix $R$ encodes the correlations. We consider a noise-dominant scenario and thus treat $R$ as the parameter to be inferred. To test joint inference and smoothing, we simulated a trajectory over an interval of $[0, 720]$ with independent Gaussian observations every $7$ units.
For optimization, we use the alternating gradient descent. As demonstrated qualitatively in Fig. \ref{img:gbm}, state and correlation structure can be inferred quite well. Note that we show the correlation matrix $RR^T$ since many $R$ may give rise to the same process. The details of the experiment and a more quantitative evaluation are given in App. \ref{ssec:joint_inference_smoothing}.

\subsection{Amortized Smoothing} \label{sec:ou_process}

We explore the possibility of amortized smoothing (Sec. \ref{sec:inference}). To keep it simple, we consider a two-dimensional Ornstein-Uhlenbeck process given by the SDE
\begin{align*}
\mathrm dX_t = -\gamma (X_t - \mu) \mathrm dt + \sigma \mathrm dW_t   \, .
\end{align*}
where $\mu \in \mathbb R^2$, $\gamma, \sigma \in \mathbb R^{2 \times 2}$. We generated 1000 trajectories of a two-dimensional model with fixed parameters and initial conditions. Each sample was observed over 20 $\mathrm{s}$  with 9 evenly spaced observation. The inference network was trained over 50 epochs using the Adam optimizer with default parameters, a weight decay of $0.001$ and a batch size of one. Fig. \ref{img:ou_process} shows the prediction of the smoothing network on a previously unseen sample compared to the exact solution. This demonstrates, in principle, that the controls for variational smoothing can be learned and that the inference network generalizes to unseen trajectories. 

\begin{figure}[ht]
\centering
\includegraphics[width=0.95\linewidth]{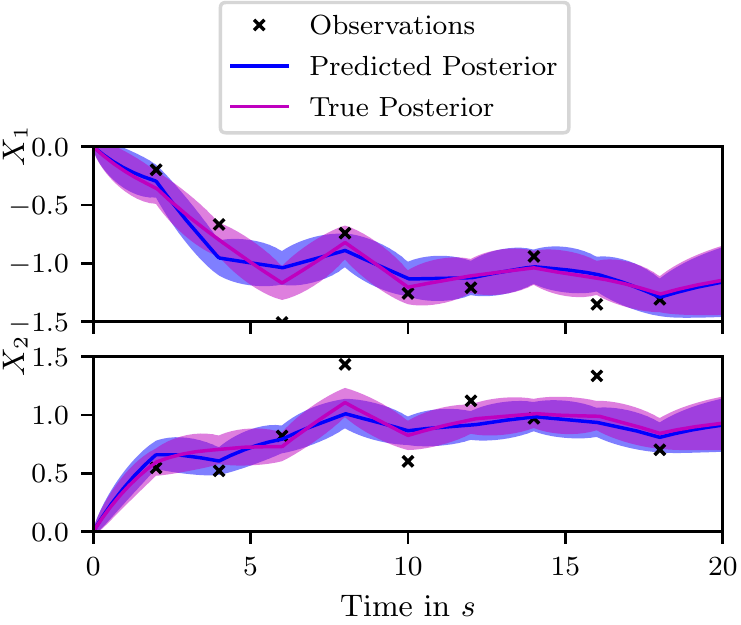}
\caption{ Two-dimensional Ornstein-Uhlenbeck process. For both components, the graph shows the smoothing predicted by the inference network on a previously unseen example. For comparison, we also show the simulated ground truth and the true posterior. Shaded regions indicate the standard deviations of the corresponding process.}
\label{img:ou_process}
\end{figure}

\subsection{Population Models}

Population models describe the time evolution of a number of species over time. A convenient way to represent a population model is via the language of chemical reactions. More precisely, let there be species $\mathcal X_1, \mathcal X_d$ and $r$ reactions of the form
\begin{equation*}
s_{i1} \mathcal X_1 + \ldots + s_{id} \mathcal  X_d \ \xrightarrow{c_i} \ p_{i1} \mathcal X_1 + \ldots + p_{id}  \mathcal X_d
\end{equation*}
with the matrices $S$ and $P$ encoding the number of molecules before and after a certain reaction event and the rate constants $c_i$ determining the time scale of each reaction. 
In addition, let $V \equiv P-S$. Then the $j$-th row $v_j$ of $V$ encodes the net change caused by reaction $j$. Under certain conditions, the concentrations of the species is governed by the chemical Langevin equation \cite{gillespie_2000}. This leads to an SDE of the form
\begin{align} \label{eq:chemical_langevin}
\mathrm d X_t =  V^T h(X_t) \mathrm \mathrm dt + \sqrt{ V^T \mathrm{diag} ( h (X_t ) )V  } \mathrm d W_t 
\end{align}
where $\sqrt{ \cdot }$ indicates a matrix square root and the mass-action propensity $h: \mathbb R^d \rightarrow \mathbb R^r$ is defined component wise by 
\begin{equation*}
h_i (x) = c_i \prod_{k=1}^d \frac{x_k !}{s_{ik}! (x_k - s_{ik})!} .
\end{equation*}
We combine a linear control and a multivariate log-normal closure to derive a general method for \eqref{eq:chemical_langevin} (see  App. \ref{ssec:population_models}). As a test system, we use the stochastic Lotka-Volterra model that describes the interaction of a prey species and a predator species. The corresponding matrices are given by 
\begin{equation*}
S = \begin{pmatrix}
1 & 0 \\
1 & 1 \\
0 & 1 
\end{pmatrix} \quad , \quad
P = \begin{pmatrix}
2 & 0 \\
0 & 2 \\
0 & 0 
\end{pmatrix}
\end{equation*}
We stress, however, that our code is not specific to the predator prey dynamics but takes general $S$ and $P$ as input. To study the behavior of our approach we recreate a scenario from \citet{ryder_2018}. We generate a synthetic trajectory starting from the initial $X_0 = (71,79)^T$ and take four observations within the interval $[0, 50]$. As shown Fig. \ref{img:lotka_volterra}, the variational smoothing can reconstruct the true trajectory quite accurately. We also observe that only four observations restrict the variance of the process significantly.

\begin{figure}[ht]
\centering
\includegraphics[width=0.95\linewidth]{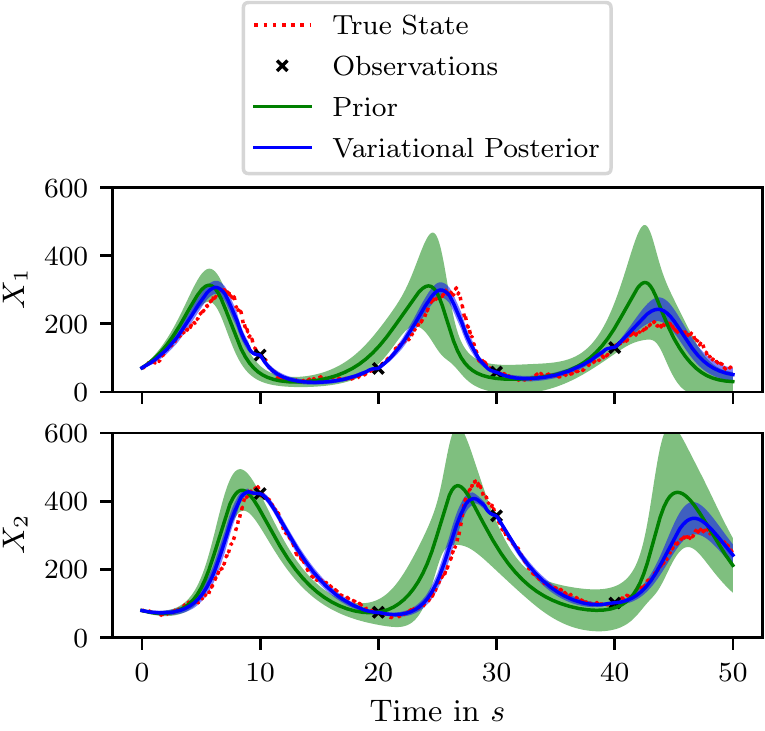}
\caption{Smoothing for the stochastic Lotka-Volterra system. Solid lines indicate the mean, the  shaded area indicates the corresponding standard deviation. Results are shown for the prior process and the variational smoothing. For reference, the simulated ground truth and the noisy observations are also provided.}
\label{img:lotka_volterra}
\end{figure}

\section{DISCUSSION}

We provide an ODE-based approach to variational smoothing that extends classical Gaussian process regression to models with state-dependent diffusion and allows for more versatile variational families. To achieve this, we understand the variational process as a controlled modification of the prior process and project the marginal posterior to a set of selected moment functions. In comparison to earlier work, we apply a refined optimization algorithm based on the natural gradient descent. 
Conceptually, our work extends a previous moment-based variational approach from MJPs to SDEs. Due to the structural similarity of both approaches, the moment-based variational method provides a unified inference framework for both process classes. In interesting future direction is to extend the moment-based variational framework to other Markov processes, in particular to jump-diffusions. 
In this work, we have used two simple closure schemes that work sufficiently well in the considered examples. Future work may consider more advanced clsoure schemes and also investigate the effect of different closures on the inference quality.  \\
While previous ODE-based approaches have required manual derivations of the backward equation and gradients with respect to the parameters, we exploit automatic differentiation to construct these quantities automatically. In general, our approach only requires to provide two model-specific functions. For certain subclasses, these functions can be constructed automatically as well. 
Since our method is gradient-based, it can implemented as an automatically differentiable function. This allows straightforward integration with deep models. As a first example, we train an amortized inference network on a toy model with known model parameters. A promising future direction is to extend this to a full variational autoencoder for time series.

\section*{Acknowledgements}

This work was supported by the European Research Council (ERC) within the CONSYN project, grant agreement number 773196.

% References follow the acknowledgements.  Use an unnumbered third level
% heading for the references section.  Please use the same font
% size for references as for the body of the paper---remember that
% references do not count against your page length total.
% \twocolumn
% \newpage

\FloatBarrier

% References
\bibliography{library}

% appendix 
\newpage 
\onecolumn
\appendix

\def\theequation{A\arabic{equation}}
\def\thesection{A\arabic{section}}
\setcounter{equation}{0}

% redefine labels
\def\theprop{A\arabic{prop}}
\setcounter{prop}{0}
\def\thethm{A\arabic{thm}}
\def\thelemma{A\arabic{lemma}}
\setcounter{lemma}{0}

\section{Proofs and Derivations}

\subsection{Proof of Lemma 1} \label{ssec:proof_lemma_1}

\paragraph{Step 1}
First, we derive representation \eqref{eq:quadratic_form}. We start from the proposed variational drift \eqref{eq:stat_control} with identity rescaling, i.e. 
\begin{align} \label{seq:stat_control}
a^Z(x, t) = a^X(x) + v(t) T(x) \, .
\end{align}
with $v:[0,T] \rightarrow \mathbb R^{n \times m}$ and $T: \mathbb R^n \rightarrow R^m$. If we denote the columns of $v$ by $v_1, \ldots, v_m$, the control part of Eq. \eqref{seq:stat_control} may be written as
\begin{align} \label{seq:stat_control_expanded}
 v(t) T(x) = \sum_{i=1}^m v_i(t) T_i(x)
\end{align}
where $T_i$ represents the components of the vector $T$. 
%In the main text, we also introduced the vectorized form of $v$. Here, we can define $u: \mathbb [0,T] \rightarrow \mathbb R^{n \times m}$ a bit more precisely bi its components $u_i: [0,T] \rightarrow \mathbb R$ with
%\begin{align} \label{seq:vectorized_control}
%u_{(i-1)m + j } (t) = v_{i, j} (t)  \, , \quad i = 1, \ldots, n \, , \quad  j = 1, \ldots, m \, . 
%\end{align}
The drift \eqref{seq:stat_control} induces a family of Ito processes parametrized by the deterministic, time-dependent function $v$ given by the SDE
\begin{align} \label{seq:variational_family}
\mathrm d Z_t = \left( a^X(Z_t) + v(t) T(Z_t)  \right)  \mathrm dt + b(Z_t) \mathrm d W_t \, . 
\end{align}
Inserting the control \eqref{seq:stat_control} into the objective function \eqref{eq:objective_function}, the prior drift $a^X$ cancels and the divergence between variational process and prior becomes
\begin{align} \label{seq:variational_family_kl}
 D_\mathrm{KL}[ P^Z \, || \, P^{ X} ]  =  \frac{1}{2}\int_0^T \mathsf E \left[ T(Z_t, t)^T v (t)^T D(Z_t, t)^{-1} v(t) T(Z_t, t) \right] \mathrm dt
\end{align}
To proceed, let us rewrite the tensor contraction within the expectation above using the expanded form of the control \eqref{seq:stat_control_expanded}.  For clarity, we also suppress the arguments and get
\begin{equation} \label{seq:tensor_contraction}
T^T v^T D^{-1} v T= \sum_{i, j}  v_{i} T_i  D^{-1} T_j v_{j} 
\end{equation}
%using Einstein sum convention. For clarity, we also suppress the arguments and get
%\begin{equation} \label{seq:tensor_contraction}
%T^T v^T D^{-1} v T= T_i v_{ji} D^{-1}_{jk} v_{kl} T_l  = v_{ji} v_{kl} T_i D_{jk}^{-1}  T_l
%\end{equation}
Now let us define a block matrix  function $\psi: \mathbb R^n \rightarrow  \mathbb R^{nm \times nm}$  and a vector $u \in \mathbb R^{nm}$ as
\begin{align*}
\psi(x) = \begin{pmatrix}
T_1 T_1 D^{-1} & \ldots & T_1T_m   D^{-1} \\
\vdots & \cdots & \vdots \\
T_m T_1 D^{-1} & \ldots, & T_m T_m D^{-1}
\end{pmatrix} \, , \quad 
u = \begin{pmatrix}
v_1 \\
\vdots \\
v_m
\end{pmatrix} \, .
\end{align*}
With this, the tensor contraction  \eqref{seq:tensor_contraction} can be written as
\begin{align} \label{seq:obj_quadratic_form}
T^T v^T D^{-1} v T = u^T \psi u \, . 
\end{align}
Now observe that only the elements of $\psi$ depend on the stochastic process $Z_t$. Consequently, we arrive at the following form of the divergence
\begin{align} \label{seq:kl_g}
D_\mathrm{KL} [ P^Z \, || \, P^{ X} ]  =  \frac{1}{2}\int_0^T u(t)^T \mathsf E[ \psi(Z_t) ] u(t)  \mathrm dt \, .
\end{align}
This function contains expectations of the form $\mathsf E[ T_i(Z_t) T_j(Z_t) D^{-1}_{kl} (Z_t) ) ]$. By augmenting the summary statics $S$ with these quantities, we can find a function $g$ such that $\psi(t) = g( \varphi(t) ) $. While this choice is always possible, there may often be more convenient ways to construct $g$. Using $g$, we can write 
\begin{align}
D_\mathrm{KL}[ P^Z \, || \, P^{ X} ]  =  \frac{1}{2}\int_0^T u(t)^T g( \varphi (t) )  u(t) \mathrm dt \, .
\end{align}
as given in \eqref{eq:quadratic_form}. 

\paragraph{Step 2} 
 What is left to show is that \eqref{seq:kl_g} is a quadratic form in $u$.  For this, we have to show that the matrix $\psi(t)$ is symmetric positive semi-definite for almost every $t \in [0, T]$. 
The function $\psi$ is symmetric by construction. Consequently, $E [\psi (Z_t )]$ is symmetric. Now fix $u \in \mathbb R^{nm}$ and let $v$ be the corresponding matrix representation. Set
\begin{align*}
H(v, t) = T(x,t)^T v(t)^T D^{-1} (x, t) v(t) T(x, t) \, . 
\end{align*}
Then $H(v, t) \geq 0$ for almost all $x \in \mathbb R^n$ and $t \in [0,t]$. To see this, we understand $v(t) T(x,t)$ as a vector in $\mathbb R^n$. The result follows because $D(x,t)$ is p.s.d. almost everywhere by assumption. Here, the strict definiteness is lost in general. For example, if $m>n$ we may find $v$ such that $v(t) T(x,t)=0$.  Now since $H$ and $A$ represent the same quadratic form, $A$ has to be positive semi-definite as well. 

\paragraph{Step 3}

So far, we have not considered rescaling. Now let $R: \mathbb R^n \rightarrow R^{n \times n}$ be an invertible matrix valued function and consider the rescaled controlled drift
\begin{align*}
a^Z(x,t) = a^X(x) + R(x) v(t) T(x,t) \, . 
\end{align*}
Evaluating the prior contribution to the divergence yields
\begin{equation} \label{sseq:rescaled_variational_family_kl}
 D_\mathrm{KL}[ P^Z \, || \, P^{ X} ]  =  \frac{1}{2}\int_0^T \mathsf E \left[ T(Z_t)^T v (t)^T \underbrace{R(Z_t)^T D(Z_t, t)^{-1} R(Z_t)}_{\equiv \tilde D^{-1}( Z_t) } v(t) T(Z_t) \right] \mathrm dt  \, . 
\end{equation}
We can now directly repeat steps 1 and 2 by replacing $D(x) \rightarrow \tilde D(x)$ in the construction of $g$.

\subsection{Proof of Proposition 1}  \label{ssec:proof_prop_1}

The full variational inference problem for the proposal class parametrized by $u$ is given by
\begin{equation}  \label{seq:variational_problem}
\min_{u}    \frac{1}{2}\int_0^T u(t)^T g( \varphi (t) )  u(t) \, \mathrm dt - \sum_{k=1}^n \mathsf E[ \log p(y_k \mid Z_{t_k} ) ] + \log C \,  . 
\end{equation}
Eq. \eqref{seq:variational_problem} is a direct consequence of \eqref{eq:objective_function} and Lemma \ref{thm:quadratic_form} and may be understood as a stochastic control problem. We now assume that the expected log-likelihood may be expressed in terms of the expected summary statistics $\varphi$, i.e.  $\mathsf E[ \log p(y_k \mid Z_{t_k} ) ] = F( \varphi (t_k) ) $  . We also ignore the evidence $\log C$ which is independent of $u$. This leads to the streamlined representation
\begin{equation}  \label{seq:variational_problem_simplified}
\min_{u}    \frac{1}{2}\int_0^T u(t)^T g( \varphi (t) )  u(t) \, \mathrm dt - \sum_{k=1}^n  F( \varphi(t_k)  ) \,  . 
\end{equation}
The objective function in \eqref{seq:variational_problem_simplified} corresponds to the negative ELBO for the proposal family parametrized by $u$. Unfortunately, the simple appearance of \eqref{seq:variational_problem_simplified} is deceiving since $\varphi(t) = \mathsf E[ S(Z_t) ] $ implicitly depends on $u$. Now recall that $\varphi(t)$ obeys an evolution equation of the form
\begin{align*}
\frac{d}{dt} \varphi(t) = \mathsf E [ A^\dagger S(Z_t) ] \,. 
\end{align*}
We can now convert \eqref{seq:variational_problem_simplified} into a constrained problem
\begin{equation}  \label{seq:variational_problem_constrained}
\begin{aligned}
\min_{u, \varphi} & &  & \frac{1}{2}\int_0^T u(t)^T g( \varphi (t) )  u(t) \, \mathrm dt -\sum_{k=1}^n F_i(\varphi(t_k) ) \\
\mathrm{s.t. } & &  &\dot  \varphi(t) =  \mathsf E [ A^\dagger S(Z_t) ]     \,  \\   
\end{aligned}  \, . 
\end{equation}
We emphasize that \eqref{seq:variational_problem_constrained} is an equivalent representation of the variational problem and does not contain any approximations beyond the choice of the variational family. This also means that it has the full complexity of the original stochastic control problem \eqref{seq:variational_problem}. In order to obtain a more tractable problem, we use a moment closure on the constraint to get an ODE of the form   
\begin{align} \label{seq:forward_equation}
\frac{d}{dt}  \varphi(t) = f(u, \varphi(t) ) \, . 
\end{align}
 With this moment closure relaxation,  the variational problem \eqref{seq:variational_problem_constrained} reduces to the deterministic control problem given in the main text. 
 
\paragraph{Comment} 
In general, when a moment closure is employed, there is no global Markov process corresponding exactly to the closed moment equations. Then, the objective function is not a true lower bound of the evidence but an approximate lower bound. Such behavior is well-known for other structured variational approximation, e.g. for cluster variational methods \cite{yedidia_2000, wainwright_2008}. Since moment closure introduces an additional approximation, results of the moment-based variational inference have to be checked empirically. However, we do not see this as a major problem since such empirical verification is required for all forms of variational inference anyway. 

\subsection{Proof of Lemma 2} \label{ssec:proof_lemma_2}

Consider to stochastic processes $Z$ and $Z'$ over an interval $[0,T]$ that are members of the variational family as defined by the drift and Eq. \eqref{eq:stat_control}.  Let  $Z$ and $Z'$ be parametrized by $u$ and $u'$, respectively. 
Inserting $Z$ and $Z'$ into the general path divergence of diffusion processes \eqref{eq:path_kl} , we get
\begin{align*}
D_\mathrm{KL}[ Z \, || \, Z' ] = \frac{1}{2} \int_0^T \mathsf E \left[ T(Z_t)^T \left( v'(t) - v(t) \right)^T D^{-1} (Z_t, t)      \left( v'(t) - v(t) \right) T(Z_t) \right] dt
\end{align*}
This expression is the same as \eqref{seq:variational_family_kl} with $v$ replaced by $v-v'$.  Therefore, we may follow along the same lines as in Sec. \ref{ssec:proof_lemma_1} and obtain
\begin{align*}
D_\mathrm{KL}[ Z \, || \, Z' ] = \frac{1}{2} \int_0^T \left( u'(t) - u(t) \right)^T g( \varphi (t) )   \left( u'(t) - u \right) dt \, . 
\end{align*}
 To get the suggested representation, set $G: \mathcal U \times \mathcal U \rightarrow \mathbb R$ as
\begin{align*} 
G(u)[ u'-u, u'-u] = \int_0^T \left( u'(t) - u(t) \right)^T g(\varphi  (t) )  \left( u'(t) - u \right) dt \, . 
\end{align*} 
What is left is to verify the properties of $G$. Bilinearity follows immediately from the above definition. The symmetry and positive definiteness follow from Lemma \ref{thm:quadratic_form}. 

\subsection{Proof of Proposition 2} \label{ssec:proof_prop_2}

Consider the variational problem of the main text given by
\begin{align} \label{seq:functional_opt}
u^* = \arg \min_u J[u] \, . 
\end{align}
Here, we understand $J$ as a functional of the controls $u$ only. This is possible because $u$ together with an initial condition fully defines the moments $\varphi$. We would like to perform a steepest descent that respects the metric of the manifold on which $u$ lives. We therefore replace \eqref{seq:functional_opt} by the sequential optimization problem 
\begin{align*}
u^{(i+1)} = \arg \min_u J[u]  \, , \quad \mathrm{s. t.} \quad \frac{1}{2} G(u^{(i)}) ( u - u^{(i)} ,  u - u^{(i)}) = \epsilon \, . 
\end{align*}
The following calculation is inspired by the discussion of neighboring optimal solutions in \citet{stengel_1994}. For sufficiently small $\epsilon$, we may linearize $J$ around the current estimate $ u^{(i)}$. Enforcing the constraint via a Lagrange multiplier, we obtain the unconstrained functional 
\begin{align*}
J' [ u , \lambda] = J[ u^{(i)} ] + \delta J [u^{(i)}, u - u^{(i)}]  + \lambda \left (  \frac{1}{2} G(u^{(i)}) ( u - u^{(i)} ,  u - u^{(i)}) - \epsilon \right) \, . 
\end{align*}
Here $\delta J$ denotes the Gateaux derivative of $J$. For simplicity, set $\delta u = u - u^{(i)}$. We consider $\delta u$ as a small perturbation of $u^{(i)}$ and denote as $\delta \varphi$ the deviation from $\varphi^{(i)}$ induced by $\delta u$. It turns out that $\delta \varphi$ satisfies the linearized forward equation
\begin{align} \label{seq:forward_linear}
\dot{ \delta \varphi(t) } = f_{\varphi }( u^{(i)} (t), \varphi^{(i)} (t) ) \delta \varphi(t) +   f_{u}( u^{(i)} (t), \varphi^{(i)} (t) ) \delta u \, . 
\end{align}
  The linearized contribution $\delta J$ is given by
\begin{align*}
 \delta J [u^{(i)}, u - u^{(i)}]  = \int_0^T  L_\varphi ( u^{(i)}(t),   \varphi^{(i)} (t) ) \delta \varphi(t) \mathrm d t +  \int_0^T  L_u ( u^{(i)}(t),   \varphi^{(i)} (t) ) \delta u(t) \mathrm d t \, 
\end{align*}
where $\varphi$ is understood as a function of $\delta u$ determined by \eqref{seq:forward_linear}. We therefore obtain the objective function
\begin{equation} \label{seq:objctive_linear}
\begin{split}
J' [ u , \lambda] &=   \int_0^T  L_\varphi ( u^{(i)}(t),   \varphi^{(i)} (t) ) \delta \varphi(t) \mathrm d t +  \int_0^T  L_u ( u^{(i)}(t),   \varphi^{(i)} (t) ) \delta u(t) \mathrm d t  \\
&\quad + \frac{\lambda}{2} \int_0^T \delta u(t)^T \psi( \varphi^{(i)} (t) )  \delta u (t) \mathrm d t + \mathrm{const} \, . 
\end{split}
\end{equation}
that we have to minimize subject to the constraint \eqref{seq:forward_linear}. We can now follow the standard variational procedure to obtain an adjoint equation
\begin{align} \label{seq:adjoint_linear}
\dot \eta (t)  = L_\varphi ( u^{(i)}(t),   \varphi^{(i)} (t) ) - f_{\varphi }^T( u^{(i)} (t), \varphi^{(i)} (t) ) \eta(t) 
\end{align}
that satisfies the reset conditions
\begin{equation} \label{seq:adjoint_reset}
\eta(t_k^- ) = \eta(t_k^+) + \frac{d}{d\varphi } F(  \varphi(t_k) )  , \quad  k = 1, \ldots, n
\end{equation}
at the observation times. In addition, we obtain the algebraic constraint
\begin{align} \label{seq:constraint_linear}
0 =  L_u ( u^{(i)}(t),   \varphi^{(i)} (t) )  - f_{u}( u^{(i)} (t), \varphi^{(i)} (t) )  \eta(t) + \lambda \psi( \varphi^{(i)} (t) ) \delta u(t)   \, . 
\end{align}
In contrast to the non-linearized problem, the stationarity conditions decouple in this case. This means we can solve for the controls explicitly. Denoting the solution of \eqref{seq:adjoint_linear} as $\eta^{(i)}$ the solution can be expressed as
\begin{align} \label{seq:constraint_solution}
\delta u^{(i)}(t) = -\frac{1}{\lambda} g^{-1}( \varphi^{(i)} (t) ) \left( L_u ( u^{(i)}(t),   \varphi^{(i)} (t) ) -  f_{u}( u^{(i)} (t), \varphi^{(i)} (t) )  \eta^{(i)} (t) \right) 
\end{align}
Now we also know that $L_u ( u^{(i)}(t),   \varphi^{(i)} (t) )  = g( \varphi^{(i)} (t) ) u^{(i} (t)$. Thus, we get the natural gradient update steps as
\begin{align} \label{seq:natural_gradient_step}
u^{(i+1)} (t) = u^{(i)} (t) + \delta u^{(i)}(t)  =  u^{(i)} (t)  - h \left( u^{(i)} (t) -  \psi^{-1}( \varphi^{(i)} (t) )   f_{u}( u^{(i)} (t), \varphi^{(i)} (t) )  \eta^{ (i) }  (t) \right) \, . 
\end{align}
Here, we also introduced the step size $h= \frac{1}{\lambda}$ that is determined by $\epsilon$. To recover the regular gradient, we use that gradient descent is a steepest descent with respect to the Euclidean metric. In our function space setting, the natural analogue is the $L_2$ norm. Therefore, gradient descent is obtained by repeating the above calculations for
\begin{align*}
G(u) (u'-u, u'-u) = \int_0^T (u'(t)-u(t))^T (u'(t)-u(t)) \, \mathrm{dt} \, .
\end{align*}
Thus, the only thing we have to change is replacing $g^{-1}$ in \eqref{seq:constraint_solution} with the identity matrix. 

\paragraph{Comments}
In a typical application, we will initialize the descent with all controls set to zero. This setting recovers the prior process. Intuitively, we can see the natural gradient descent as a smooth transition from the prior process to the (locally) best posterior approximation within the variational family. We note that due to moment closure, we only have access to approximate moments $\varphi$. Therefore, the natural gradient is also an approximation to the true natural gradient. 

\subsection{Recovering the Gaussian Process Approximation} \label{ssec:vgpa}

For this section, we assume that the diffusion term $b$ does not depend on the state. The Gaussian process approximation only requires first and second order moments. We therefore choose 
\begin{align*} 
S(x) = (x_1, \ldots, x_n, x_1^2, x_1 x_2 , x_2^2, \ldots,  x_n^2 )^T \, . 
\end{align*}
The GP approximation is defined by a linear time-dependent drift. In order to recover this within our framework, we need to find $T$ such that
\begin{align*}
a^Z(x,t) = a(x) + v (t) T(x) \stackrel{!}{=}  u_0(t) + u_1 (t) x
\end{align*}
where $u_0: [0,T] \rightarrow \mathbb R^n$ and $u_1: [0,T] \rightarrow \mathbb R^{n \times n}$.  Now consider the choice
\begin{align*}
T(x,t) = \begin{pmatrix}
1 \\
x \\
a(x, t) 
\end{pmatrix}  \, , \quad
v(t) = \begin{pmatrix}
u_0(t) & u_1(t) & u_2 (t)  
\end{pmatrix}
\end{align*}
understood as block matrix notation with $u_2 : [0, T] \rightarrow \mathbb  R^{n \times n}$. This will lead to a variational drift of the form
\begin{align*}
a^Z(x,t) =  a(x) + u_0(t) + u_1 (t) x  + u_2(t) a(x)  \, . 
\end{align*}
If we fix $u_2$ to the constant function with output minus one, the prior drift will cancel and we have constructed a linear GP.  Using the general moment equation \eqref{eq:moment_equation} with the drift $a^Z$, we can now derive the moment equations for $\varphi(t) = \mathsf E[ S(Z_t) ] $. Represented in terms represented in terms of $m$, $\bar M$ we get
\begin{equation} \label{seq:lgp_moments}
\begin{split}
\dot m(t) &= u_0(t) + u_1(t) m(t) \, , \\
\dot{ \bar M}(t) & = u_1(t) \bar M(t) + \bar M(t) u_1(t)^T  + D  \, . 
\end{split}
\end{equation}
These equations are the standard equations for mean and variance of a linear GP. Eq. \eqref{seq:lgp_moments} defines the forward function $f$ required for implementation in our framework. The second function required is $L$ or $g$. Here, $L$ is a bit more convenient and is given by
\begin{align*}
L( u(t), \varphi(t) ) = \mathsf E[ (u_0(t) + u_1 (t) x  -  a(Z_t) )^T D^{-1}  (u_0(t) + u_1 (t) Z_t  -  a(Z_t) ) ]\, . 
\end{align*}
After a few algebraic multiplications, we observe that the only model dependent quantities required for $L$ are $\mathsf E[a(Z_t]$, $\mathsf E[a(Z_t) Z_t^T] $ and $\mathsf E[ a(Z_t) a(Z_t)^T]$ expressed as functions of $m$ and $\bar M$.

\section{Additional Information}

\subsection{Moment Closure Approximations} \label{ssec:moment_closure}

In the main part, we discussed how two obtain moment equations for Markov processes and gave a general idea on moment closure. Here, we will discuss strategies to obtain such a closure scheme, i.e. how to find the function $h$ such that we can proceed from \eqref{eq:moment_dynamics} to \eqref{eq:closed_dynamics}.  We focus on distributional closures that correspond to a projection onto a given parametric family \citep{bronstein_2018}. A distributional closure is constructed by picking a parametric proposal distribution $q_{ \phi}$ on the state space $\mathcal X$. To obtain a closure scheme, the first step is to find $\phi(t)$ such that
\begin{align} \label{seq:distributional_parameter}
\mathsf E_{\phi(t) } [ R( X) ] = \varphi (t)  \, , 
\end{align}
 where $\varphi(t) = \mathsf E[S(Z_t) ]$ are the expected summary statistics used to approximate the process. Eq. \eqref{seq:distributional_parameter} defines a valid moment closure when the conditions of the implicit function theorem are satisfied. Assuming we have obtained $\varphi(t)$, we can evaluate
\begin{align} \label{seq:distributional_closure}
  h( \varphi (t) ) \equiv \mathsf E_{\phi(t) } [ R( X) ] \, ,
\end{align}
where the expectation is taken with respect to $q_{\phi(t)}$.

It has also been shown that moment closure tends to work better if the support of the proposal distribution matches the support of the target process. Here, we focus on two simple probabilistic closures: the multivariate normal closure for processes defined on $\mathbb R^n$ and the multivariate log-normal closure for processes defined on $\mathbb R_+^n$. Another advantage of these two schemes is that they correspond directly to first and second order moments and may thus be used as a starting point before investigating more advanced schemes. To simplify the presentation, we denote the first order moments as $m \equiv \mathsf E[X]$,  the second order moments as $M \equiv \mathsf E[ X X^T]$ and the second order central moments as $\bar M \equiv \mathsf E[ (X-m) ( X-m)^T ]$. We write general powers in multi-index form, i.e.
\begin{align*}
X^\alpha \equiv \prod_{i=1}^n X_i^{\alpha_i}  \, . 
\end{align*}
This will be useful to represent general power moments. We also define $ k \equiv \sum_{i=1}^n \alpha_i $ as the order of the $\alpha$-th multi-moment.

\paragraph{Multivariate Normal Clouse}
Let $X \sim \mathcal N(\mu, \Sigma)$ be a multivariate normal random variable on $\mathbb R^n$. Since the distribution is fully specified by the man $\mu$ and the covariance $\Sigma$, all moments of the form $\mathsf E[ X^\alpha ] $ with $\alpha \in \mathbb N^n$ can be expressed as functions of $\mu$ and $\Sigma$. One way to obtain such a representation is via Isserlis' theorem \citep{isserlis_1918}. We will follow an alternative approach via the moment generating function.
\begin{lemma} \label{sthm:normal_closure}
  Let $X$ be as above and $\alpha \in \mathbb R^n$ a multi-index. Then 
  \begin{align*}
    \mathsf E[ X^\alpha] =  \left. \prod_{i=1}^n \frac{\partial^{\alpha_i}}{\partial s_i^{\alpha_i}} M(s) \right|_{s=0}
  \end{align*}
  where
  \begin{align*}
    M(s) = \mathsf E[ \exp( s^T X)] = \exp \left( \mu^T s + \frac{1}{2} s^t \Sigma s \right)
  \end{align*}
  is the moment generating function of $\mathcal N(\mu, \Sigma)$.
\end{lemma}
While Lemma \eqref{sthm:normal_closure} does not provide an explicit formula for direct numerical implementation, it is straightforward to automatically generate the closure relations using a computer algebra system. In particular, we automatically construct moment equations using Sympy \citep{meurer_2017} and convert them to PyTorch functions.

\paragraph{Multivariate Lognormal Closure}

A log-normal random variable can be obtained by exponential transform of a normal random variable. This generalizes to the multivariate case. More specifically, let $Z \sim \mathcal N( \mu, \Sigma)$. Then we say
\begin{align*}
X = \exp( Z) \, 
\end{align*}
has a log-normal distribution. Here, the exponential is understood as acting component-wise.
\begin{lemma} \label{sthm:lognormal_closure}
Let $X$ be as above and $\alpha \in \mathbb R^n$ a multi-index. Then
\begin{align*}
\mathsf E[ X^\alpha ]  &= \left( \prod_{i} \frac{ \mathsf E[ X_i]^{2 \alpha_i}  }{ \mathsf E[ X_i^2 ]^{\frac{\alpha_i}{2} } }  \right) \left( \prod_{i,j}    \frac{ \mathsf E[ X_i X_j]^{\frac{\alpha_i \alpha_j}{2} }  }{ \mathsf E[ X_i ]^{\frac{\alpha_i \alpha_j}{2} }  \mathsf E[X_j]^{\frac{\alpha_i \alpha_j}{2} }   }   \right) \, 
\end{align*}
\end{lemma}
%\begin{proof}
% By definition of the multivariate lognormal distribution, we have
%\begin{align*}
%\mathsf E[ X^\alpha ] &= \mathsf E \left[ \exp \left(  \alpha^T Z \right) \right] \, ,
%\end{align*}
%where $Z \sim \mathcal N( \mu, \Sigma)$. The r.h.s. expression corresponds to the moment generating function of a multivariate normal distribution. Thus, we get
%\begin{align} \label{seq:lognormal_multi_moment}
%\mathsf E[ X^\alpha ] &= \exp \left( \mu^T \alpha + \frac{1}{2} \alpha^T \Sigma \alpha  \right)  \, . 
%\end{align}
%The remaining task is to express $\mu $ and $\Sigma$ in terms of first and second order moments. For this, we observe that
%\begin{align*}
%\mathsf E[ X_i ] &= \exp \left( \mu_i + \frac{1}{2} \Sigma_{ii} \right)  \, , \\
%\mathsf E[ X_i X_j ] &= \exp \left( \mu_i + \mu_j  + \frac{1}{2} ( \Sigma_{ii} +  \Sigma_{ij} + \Sigma_{ji} + \Sigma_{jj} )  \right)  \\
%&\quad = \mathsf E[ X_i ] \mathsf E[ X_j ] \exp( \Sigma_{ij} )
%\end{align*}
%By inverting this system of equations we get
%\begin{align*}
%\exp(\mu_i) &= \frac{ \mathsf E[ X_i]^{2}  }{ \mathsf E[ X_i^2 ]^{\frac{1}{2} } }  \\
%\exp( \Sigma_{ij} ) &=  \frac{ \mathsf E[ X_i X_j]  }{ \mathsf E[ X_i ] \mathsf E[X_j]  }  \\
%\end{align*}
%The claim follows by splitting the exponential in \eqref{seq:lognormal_multi_moment} and inserting the last result. 
%\end{proof}
This result can be shown by exploiting that $\mathsf E[ X^\alpha ] = \mathsf E \left[ \exp \left(  \alpha^T Z \right) \right]$ corresponds to the moment generating function of a normal distribution. As a consequence of Lemma \ref{sthm:lognormal_closure}, we obtain the explicit closure function
\begin{align} \label{seq:lognormal_closure_function}
\mathrm{Cl} ( m, M, \alpha) =  \left( \prod_{i} \frac{ m_i^{2 \alpha_i}  }{ M_i^{\frac{\alpha_i}{2} } }  \right) \left( \prod_{i,j}    \frac{ M_{ij}^{\frac{\alpha_i \alpha_j}{2} }  }{ m_i^{\frac{\alpha_i \alpha_j}{2} }  m_j^{\frac{\alpha_i \alpha_j}{2} }   }   \right) \, . 
\end{align}
The log-normal closure \eqref{seq:lognormal_closure_function} can be implemented efficiently using tensor operations. It is also differentiable and thus suitable for backpropagation.

\subsection{Rescaling} \label{ssec:rescaling}

Consider a drift without rescaling of the form \eqref{seq:stat_control}. While leading to a convenient quadratic objective function, this form of the control has two major drawbacks. The first drawback is that the matrix-valued function $g$ in \eqref{eq:quadratic_form} is of dimension $(n\cdot m, n \cdot m)$ with $n$, $m$ corresponding to the dimension of the state space and the number of control features, respectively. Assuming the number control features is  proportional to the dimension, the function $g$ requires $\mathcal O(n^4)$ elements.  The second problem is specific to models with state-dependent diffusion term. In this case, the elements of $g$ are of the form $\mathsf E[ T_i(Z_t) T_j(Z_t) D^{-1}_{kl} (Z_t) ]$. This expression requires an analytic inverse of the diffusion tensor which is rarely available. We now discuss two special choices of rescaling. First, consider a case where the diffusion term $b$ is known analytically and set $R = b$. We then have
\begin{align*}
\tilde D^-1(x)  \equiv b(x)^T D^{-1} (x) b(x) = b(x)^T \left( b(x) b(x)^T \right)^{-1} b(x) = I \, . 
\end{align*}
The matrix $g$ now only depends on moments of the form $\mathsf E[ T_i (Z_t, t) T_j (Z_t, t)]$. As a consequence, the objective function becomes independent of the diffusion tensor. In addition, the number of non-zero elements of $g$ now scales as $\mathcal O(n^2)$.  The second choice of rescaling aims at the case where we are provided with $D$ rather than $b$ and thus consider $R = D$ leading to
\begin{align}
\tilde D^{-1}(x) = D(x)^T D^{-1} (x) D(x) = D(x,t) \, . 
\end{align}
While this choice does not improve the scaling, $g$ now depends on expressions of the form $\mathsf E[ T_i(Z_t) T_j(Z_t) D_{kl} (Z_t) ]$ such that we get rid of the inverse diffusion tensor. It also has an interesting intuitive interpretation. Recall that the true posterior drift is given by $\bar a(x,t) = a(x,t) + D(x,t) \nabla \log( \beta(x,t) ) $. Thus, the ansatz corresponds to approximating the log-transformed backward function by a linear feature model. 

\subsection{A Standard Approximation} \label{ssec:standard_approximation}

Inspired by the variational Gaussian process approximation, we would like to construct a method that approximates the data-driven term by a feedback control linear in the state and requires first and second order moments. This corresponds to the choices
\begin{equation*}
\begin{split}
S(x) &= (x_1, \ldots, x_n, x_1^2, x_1 x_2 , x_2^2, \ldots,  x_n^2 )^T \, , \\
T(x) &= (1, x_1, \ldots, x_n )^T \, . 
\end{split}
\end{equation*}
For this special class, it is convenient to represent the control in terms of functions $u_0$, $u_1$ such that we can write
\begin{align}
v(t) T(x) = u_0(t) + u_1(t) x
\end{align}
A short calculation shows that the moment equations are given by
\begin{equation*}
\begin{split}
\dot m(t) &= \mathsf E[ a(Z_t) ] + \mathsf E[ R(Z_t) ] u_0(t) + \mathsf E( R(Z_t) u_1(t) Z_t ]  \\ 
\dot { M}(t) &= E[ a(Z_t) Z_t^T] + E[ Z_t a(Z_t)^T] + \mathsf E[ D(Z_t) ]  \\
&\quad  + \mathsf E[ R(Z_t) u_0 Z_t^T ]  + \mathsf E[ Z_t u_0 Z_t^T R(Z_t)^T ] \\
&\quad + \mathsf E[ R(Z_t) u_1(t) Z_t Z_t^T ] + \mathsf E[ Z_t Z_t^T u_1(t)^T  R(Z_t)^T  ]
\end{split}
\end{equation*} 

\subsection{Specific Subclasses} \label{ssec:examples}

In this section, we present moment equations for special classes and more specific models considered in the experiment section. 

\paragraph{Constant Diffusion}

For models with constant diffusion terms $b$, we can always choose the corresponding rescaling. In combination with the approach outlined in \ref{ssec:standard_approximation}, we obtain the moment equations 
\begin{equation*}
  \begin{split}
  \dot m(t) &= \mathsf E[ a(Z_t) ] + b u_0(t) + b u_1(t) m(t) ]  \, , \\ 
  \dot M(t) &= E[ a(Z_t) Z_t^T] + E[ Z_t a(Z_t)^T] +  b b^T \\
  &\quad + b u_0(t) m(t)^T + m(t) u_0(t)^T b^T +  b u_1(t) M (t)  + M(t) u_1 (t)^T b^T   \, .
  \end{split}
\end{equation*} 
The second function required is the contribution to the KL-divergence that can be provided in terms of $L$ or $g$ (c.f. Thm. \ref{thm:control_problem}). Here, using $L$ is more convenient and we get
\begin{align*}
  L(u(t), m(t), M(t)) = \frac{1}{2} \left( u_0(t)^T u_0(t) + 2 u_0(t)^T u_1(t) m(t) + \mathrm{Tr} (u_1(t)^T u_1(t) M(t) )  \right) \, . 
\end{align*}
Since the last equation is model-independent, we have implemented a base class from which custom models can be derived. In particular, implementation of a given model only requires custom functions for $\mathsf E[ a(Z_t) ]$ and $E[ a(Z_t) Z_t^T]$. For polynomial drift functions, the required expectations can be evaluated straightforwardly using our Gaussian moment closure implementation in Sympy (c.f. Lemma \ref{sthm:normal_closure}).

\paragraph{Population Models} \label{ssec:population_models}
We consider a general population model defined by the chemical Langevin equation \eqref{eq:chemical_langevin}. Here, the diffusion term $b(x)$ is not given and we have only access to the diffusion tensor $D(x) = V^T \mathrm{diag}(h(X_t)) V$. We will therefore choose a linear control with rescaling $R = D$. In the regime where the CLE is valid, we typically have $X_i \gg 1$. We will also restrict the discussion to physically plausible systems with $s_{ik} \leq 2$.  Under these conditions, the propensity functions can be approximated as
 \begin{align} \label{eq:approximate_propensity}
h_i(x) = c_i \prod_{k=1}^d x_k^{s_{ik}} \, . 
\end{align}
For the proposed variational process class, we obtain the moment equations
\begin{equation}
\begin{split}
m_i(t) &=  V_{ji}\mathsf E[ h_j(Z_t) ] + V_{ki}  V_{kj} u_{0, j}(t)  \mathsf E[ h_k (Z_t) ] + V_{li} V_{lj} u_{1, jk}(t) \mathsf E[h_l(Z_t) Z_{t, k}] \\
M_{ij}(t) &=  V_{ki} \mathsf E [ h_k (Z_t) Z_{t, j} ]  + V_{kj} \mathsf E [ h_k (Z_t) Z_{t, i} ]  + V_{ki} V_{kl} u_{0,l} (t) \mathsf E[ h_k(Z_t) Z_{t, j} ]  + V_{kj} V_{kl} u_{0,l} (t) \mathsf E[ h_k(Z_t) Z_{t, i} ]  \\
&\quad + V_{ki} V_{kl} u_{1, lm} (t) \mathsf E[ h_k(Z_t) Z_{t,m} Z_{t, j} ] + V_{kj} V_{kl} u_{1, lm} (t) \mathsf E[ h_k(Z_t) Z_{t,m} Z_{t, i} ] +  V_{ki} V_{kj} \mathsf E[ h_k(Z_t) ] 
\end{split}
\end{equation}
While these expressions may look unwieldy, we observe that when the propensities are modeled by \eqref{eq:approximate_propensity}, all expectations in \eqref{eq:approximate_propensity} are of the form $\mathsf E[ Z_t^\alpha]$ and can be easily evaluated with the generic log-normal closure function \ref{seq:lognormal_closure_function}. In addition, expectations required for the function $g$ are of the same form. 

\paragraph{Multivariate Geometric Brownian Motion}

For the multivariate Brownian motion model described in the main text, we consider a linear control with rescaling $b$. We obtain the moment equations
\begin{equation} \label{seq:mgbm_moments}
\begin{split}
\dot m_i(t) &= r_i m_i(t) +  m_i(t)  R_{ij}  u_{0, j} (t) +R_{ij} u_{1, jk} M_{ik}(t)  \, , \\
\dot M_{ij} &= r_i M_{ij}(t) + r_j M_{ji} (t) + R_{ik} u_{0,k} M_{ij}(t) + R_{jk} u_{0,k} M_{ji}(t) \\
&\quad  + R_{ik} u_{1, kl} \mathsf E[ Z_{t, i} Z_{t, j} Z_{t,l} ] +  R_{jk} u_{1, kl} \mathsf E[ Z_{t, i} Z_{t, j} Z_{t,l} ]   + D_{ij} M_{ij} (t)
\end{split}
\end{equation}
where we have used Einstein sum convention. In order to obtain closed equations, we compute the third order moments $ E[ Z_{t, i} Z_{t, j} Z_{t,l} ] $ via the general log-normal closure formula given in Lemma \ref{sthm:lognormal_closure}. Since for the rescaling with $b$, the second function $g$ becomes independent of the model,  \eqref{seq:mgbm_moments} is the only model specific quantity required for implementation. 

\subsection{Parameter Inference}

\subsubsection{Extended Control Problem} \label{ssec:extended_control_problem}

Variational parameter inference corresponds to maximizing the evidence lower bound jointly with respect to the model parameters $\theta$ and variational parameters $u$. If we instead minimize the negative ELBO, we obtain the optimization problem
\begin{align} \label{seq:extended_objective_function}
\min_{\theta, u} \quad   \underbrace{   D_\mathrm{KL} [ P^Z \, || \, P^X ]  - \sum_{k=1}^n \mathsf E[ \log p(y_k \mid Z_{t_k} ) ] }_{\equiv J[\theta, u]} \,  . 
\end{align}
We can now do the same reductions as for the derivation of the smoothing control problem
\begin{equation}  \label{seq:extended_variational_problem_constrained}
\begin{aligned}
\min_{\theta, u, \varphi} & &  & \frac{1}{2}\int_0^T u(t)^T g( \theta, \varphi (t) )  u(t) \, \mathrm dt -\sum_{k=1}^n F_i(\varphi(t_k) ) \\
\mathrm{s.t. } & &  &\dot  \varphi(t) =  \mathsf E [ A^\dagger S(Z_t) ]     \,  \\   
\end{aligned}  \, . 
\end{equation}
Of course, the dynamic constraint depends on $\theta$ as well. Three methods are commonly used to solve the variational inference problem \eqref{seq:extended_objective_function}.  
\paragraph{Coordinate Descent} Starting from an initial guess $\theta^{(0)}$, $u^{(0)}$ the updates are computed as
\begin{align*}
u^{(i+1)} &= \min_u J[\theta^{(i)}, u] \, , \\
\theta^{(i+1)} &= \min_\theta J[ \theta ,  u^{(i+1)} ]  \, . 
\end{align*}
This is the classical variational expectation maximization(VEM) algorithm used in mean field variational inference. It is most effective when the updates can be computed in closed form. In the scenario considered here, this will typically not be the case. Even if possible, obtaining the closed form updates requires model specific calculations that we try to avoid. We therefore do not consider the VEM any further. 
\paragraph{Gradient Descent}
In Prop. 2, we have presented regular and natural gradient descent in the controls to solve the smoothing problem. The proof can be extended straightforwardly to include a gradient update with respect to $\theta$ as well. More explicitly, the parameter updates take the form
\begin{align} \label{seq:parameter_gradient}
\theta^{(i+1)} = \theta^{(i)} - h\int_0^T \left(  L_\theta ( \theta^{(i)}, u^{(i)}(t),   \varphi^{(i)} (t) ) -  f_{\theta}(  \theta^{(i)}, u^{(i)} (t), \varphi^{(i)} (t) )  \eta^{(i)} (t) \right)  \mathrm{d} t \, 
\end{align}
with notational conventions as in Prop. 2. Eq. \eqref{seq:parameter_gradient} corresponds to a regular gradient and can be evaluated without model specific derivations based on automatic differentiation.

\paragraph{Alternating Gradient Descent}
A third alternative is to iteratively take a number of gradient steps for the $u$ and $\theta$ while keeping the other fixed (see Algorithm \ref{alg:agd}). This method has two main advantages over the full gradient descent. First, we can use separate step sizes for the model and variational parameter updates. Second, since $\theta$ is fixed for the descent in $u$, we can still exploit the natural gradient for the latter.  The alternating gradient descent can be seen as a hybrid between VEM and gradient descent. This is because if we performed every inner gradient descent up to convergence, the result would be equivalent to the VEM updates.

\begin{algorithm}
  \caption{Robust Alternating Gradient Descent for Moment-Based Variational Inference}
  \label{alg:agd}
\begin{algorithmic}[1]
  \STATE {\bfseries Input:} Initial guesses $\theta^{(0)}$, $u^{(0)}$, initial condition $\varphi(0)$, learning rates $h_0, h_1$.
  \FOR{ $i = 0, \ldots, i_\mathrm{max} $}
  \STATE $u^{(i, 0)} \rightarrow  u^{(i)}$
  \FOR{$k = 0, \ldots, k_\mathrm{max} $ }
  \STATE Set $u'$ according to \eqref{seq:natural_gradient_step}.
  \IF{ $J[ \theta^{(i)}, u' ] < J[\theta^{(i)}, u^{(i,k)} ] $ }
  \STATE  $h_0 \rightarrow \alpha \cdot h_0$, $u^{(i, k+1)} \rightarrow u'$
  \ELSE 
  \STATE $h_0 \rightarrow \beta \cdot h_0$,  $u^{(i, k+1)} \rightarrow u^{(i, k)}$
  \ENDIF
  \ENDFOR
  \STATE  $u^{(i+1)} \rightarrow u^{(i, k_\mathrm{max} )}$ 
  \STATE $\theta^{(i, 0)} \rightarrow  \theta^{(i)}$
  \FOR{ $k = 0, \ldots, k_\mathrm{max}$ }
  \STATE Set $\theta'$ according to \eqref{seq:parameter_gradient}.
  \IF{ $J[ \theta', u^{(i+1)} ] < J[\theta^{(i, k)}, u^{(i+1)} ] $ }
  \STATE  $h_1 \rightarrow \alpha \cdot h_1$, $\theta^{(i, k+1)} \rightarrow \theta'$
  \ELSE 
  \STATE $h_1 \rightarrow \beta \cdot h_1$,  $\theta^{(i, k+1)} \rightarrow \theta^{(i, k)}$
  \ENDIF
  \ENDFOR
  \STATE  $\theta^{(i+1)} \rightarrow \theta^{(i, k_\mathrm{max} )}$ 
  \ENDFOR
  \STATE {\bfseries Output:}  $\theta^{(i_\mathrm{max})}$, $u^{(i_\mathrm{max})}$
\end{algorithmic}
\end{algorithm}

\subsubsection{Amortized Inference} \label{ssec:amortized_inference}

In many scenarios ones observes not a single trajectory but a number of trajectories $\mathbf x_1, \ldots, \mathbf x_n$ produced independently from the same model underlying model with parameter $\theta$. Denote the corresponding noisy observations as $\mathbf y_1, \ldots, \mathbf y_N$. We use boldface to indicate that $\mathbf x_i$, $\mathbf y_i$ corresponds to trajectories of stochastic processes. However, to keep the following discussion simple, we will treat them informally as ordinary random variables. Thus, the joint data likelihood is given by
\begin{align*}
p( \mathbf y_1, \ldots, \mathbf y_n \mid \theta ) = \prod_{i=1}^N p_i( \mathbf y_i \mid \theta ) =  \prod_{i=1}^N \int  p_i( \mathbf y_i \mid \mathbf x_i ) p_i( \mathbf x_i \mid \theta ) \mathrm d \mathbf x_i \ \, 
\end{align*}
where $\mathbf x_i$ corresponds to the trajectory of the latent diffusion process in this case. The standard VI approach is to construct an evidence lower bound based on the proposal
\begin{align*}
q( \mathbf x_1, \ldots \mathbf x_n  ) = \prod_{i=1}^n q_i( \mathbf x_i \mid u_i ) \, . 
\end{align*}
In our case, every $q_i$ corresponds to a full stochastic process $Z_i$ parametrized by $u_i$ with $u_i$ being a function of time. Consequently, the joint variational inference problem becomes infeasible very quickly due to large memory and runtime requirements. We therefore use an amortized proposal of the form
\begin{equation*}
q( \mathbf x_1, \ldots \mathbf x_n  ) = \prod_{i=1}^n q_i( \mathbf x_i \mid h(y_i, \phi) ) 
\end{equation*}
where $h$ is a feed-forward neural network parametrized by $\phi$. The corresponding ELBO is given by
\begin{align}
\mathrm{ELBO}(\theta, \phi) = \sum_{i=1}^N \int q_i(\mathbf x_i \mid h(\mathbf y_i, \phi) ) \log p(\mathbf y_i \mid \mathbf x_i) - D_\mathrm{KL} [ q_i \, || \, p_i ] \, .
\end{align}
Now observe that every term in the sum above corresponds the objective function of the variational inference problem for a single trajectory. We may thus write
\begin{align} \label{eq:amortized_elbo}
-\mathrm{ELBO}(\theta, \phi) = \sum_{i=1}^N J[ \theta, h(\mathbf y_i, \phi) ] 
\end{align}
with $J[\theta, u]$  as defined in Eq. \ref{seq:extended_objective_function}. The key observation is now that $J$ is a scalar function and we are able to compute gradients with respect to both $\theta$ and $u$. We can therefore encapsulate the computation of $J$ and its gradients in a PyTorch module. This allows to use standard stochastic optimizers based on backpropagation to learn the model parameters $\theta$ and the inference network parameters $\phi$.

\section{Experiment Details} \label{ssec:experiment_details}

In this section, we provide the explicit equations for the examples discussed in the main text and give more detail regarding the experiments. 
\paragraph{Computational Resources}
Most experiments were run on a MacBook Pro, 2015 edition, using 2.7 GHz Intel Core i5 processor with 2 cores. We will refer to this setup as machine A. Some of the longer experiments  were run on machine B; an Intel Xeon E5-2680 v3 with 2,5GHz and 22 cores. Experiments for single trajectories were generally run on machine A. 

\subsection{Joint Inference and Smoothing} \label{ssec:joint_inference_smoothing}

We test our method with a multivariate geometric Brownian motion of dimension $n=4$. The system parameters used to generate the trajectory are given by
\begin{align*}
r = 10^{-4} \cdot \begin{pmatrix}
 1.0 \\
 2.64 \\
  1.5 \\
  3.2
\end{pmatrix}  \quad, \quad
\sigma = \begin{pmatrix}
0.0112 \\
 0.0102 \\
  0.0174 \\
  0.0130
\end{pmatrix}
\quad , \quad
\bar D = \begin{pmatrix}
1  &       -0.08 & -0.36  & 0.28 \\
 -0.08 &  1.    &      0.15 & -0.12 \\
 -0.36 &  0.15 & 1.   &  -0.52 \\
 0.28 & -0.12 & -0.52 & 1.    
\end{pmatrix} \, . 
\end{align*}
Here, $\sigma = \sqrt{ \mathrm{diag} (R R^T) }$ and $\bar D$ represents the correlation matrix obtained from normalizing $RR^T$ by $\sigma \sigma^T$. 
The simulation was started with an initial $X_0 = (1, 1, 1, 1)^T$. The corresponding $R$ was obtained from $D$ using a Cholesky decomposition. In this parameter regime, the system is noise-dominated. The parameters $r$ are thus not identifiable and we focus on recovering the correlation structure. The systems was observed over an interval $[0, 720]$ with observations every $7$ units. The observations were corrupted with Gaussian observation noise that acted independent on all components with a standard deviation $\sigma_\text{obs} = 0.01$. As a variational process class, we used second order moments with diffusion-rescaled linear control. The required equations are given in Sec. \ref{ssec:examples}.  Optimization was performed using Alg. \ref{alg:agd} with $i_\mathrm{max} = 50$ and $k_\mathrm{max} = 5$. The noise matrix was initialized as $R^{(0)} = 10^{-2} \cdot I$ corresponding to a correlation free process. In the main paper, we have shown the result of a single experiment. Here, we generate $n=100$ trajectories from the described model but use shorter trajectories over an interval $[0, 360]$. On each of these samples, we performed joint smoothing and inference. This experiment was run on machine B. We used multiprocessing with a pool of 15 workers to speed up the processing. Below, we give the average results for $\sigma*$ and $\bar D*$ along with corresponding standard deviation 
\begin{align*}
\sigma^* = \begin{pmatrix}
0.0105 \pm 0.002 \\
0.0098  \pm  .001  \\
 0.0156 \pm .002 \\
0.0118 \pm  0.001
\end{pmatrix}
\quad , \quad
\bar D^* = \begin{pmatrix}
 1    &     -0.03 \pm 0.15 & -0.31\pm 0.14  & 0.23 \pm 0.14 \\
 -0.03 \pm 0.15  & 1   &    0.13 \pm 0.13  & -0.08 \pm 0.14 \\
 -0.31 \pm 0.14 &  0.13  \pm 0.13 & 1    &   -0.46 \pm 0.11 \\
  0.23 \pm 0.14  & -0.08 \pm 0.14 & -0.46 \pm 0.11 & 1       
\end{pmatrix} \, . 
\end{align*}
Although the estimates of the correlation matrix show a bit of variation, the experiments demonstrates  that the inferred parameters are reasonable and fairly consistent across samples.

\subsection{Amortized Inference}

%Consider the Ornstein-Uhlenbeck process governed by the SDE
%\begin{align*}
%dX_t = - \gamma ( X_t - \mu ) \mathrm d t + \sigma \mathrm d W_t
%\end{align*}
%as presented in Sec. \ref{sec:ou_process}. We choose first and second order moment dynamics and a linear control with diffusion rescaling. This leads to the moment equations
%\begin{align*}
%\dot m (t) &= - \gamma ( m(t) - \mu )  + \sigma u_0(t) + \sigma u_1(t) m(t) \, ,   \\
%\dot M(t) &= -2\gamma M(t) + u_1(t) M(t) + M(t)  u_1(t)^T + \sigma \sigma^T \, . 
%\end{align*}
We follow the approach described in Sec. \ref{ssec:amortized_inference}. In order to construct the inference network, $u$ was restricted to a piece-wise constant function on an equidistant grid. This allows to represent $u$ as a matrix with rows and columns corresponding to the number of controls and the size of the time grid. In the specific experiment, the input layer consists of 16 unites corresponding to the 8 observations of two species. The input layer is followed by 6 ReLu-Linear layers of increasing  size with the final layer matching the dimensions of the control matrix. In order to train the model, we generated 1000 trajectories using the following parametrization
\begin{align*}
\gamma =  \begin{pmatrix} 
0.3 &  0 \\
 0 & 0.4 
 \end{pmatrix} \, , \quad 
 \sigma = \begin{pmatrix}
 0.2 &  0.1 \\
 0.1 &  0.15
 \end{pmatrix} \, , \quad
  \mu= \begin{pmatrix}
 -1.0 \\
 1.0
 \end{pmatrix} \, . 
\end{align*}
Observations were generated with independent Gaussian noise ($\sigma = 0.2$) on a time grid $t_\mathrm{obs} = (2.0, 4.0, 6.0, 8.0, 10.0, 12.0, 16.0, 18.0)$. In order to speed up training, we used a mini-batch size of 15. Since gradient computation for each sample requires a forward and a backward ODE integration, we implemented the mini-batch approach using multi-processing, such that each sub-process processed one sample in the usual way.

% \subsection{Data Efficiency and Extrapolation}

% Consider the standard approximation with first and second order moments and linear control (see Sec. \ref{ssec:standard_approximation}). The main difference compared to the Gaussian process approximation is that the prior drift is contained in our approach. This can have advantages in terms of data efficiency and extrapolation power. To demonstrate this, we investigated the behavior of the variational smoother depending on the number of observations provided on the Lorenz '63 model, which is given by drift
% \begin{align*}
%   a(x) = \begin{pmatrix}
%   \theta_1 (x_2 - x_1) \\
%   \theta_2  x_1 - x_2 - x_1 x_3 \\
%   x_1x_2 - \theta_3 x_3
%   \end{pmatrix} 
% \end{align*}
% and a diagonal driving noise $b$ with $b_{ii} = \sqrt{10}$. As shown in Fig. \ref{simg:lorenz63}, a single observation early in the trajectory is sufficient to predict the trajectory over the range of interest. In contrast, the Gaussian process approximation gets stuck in local minima during the optimization as long as the number of observations is small. In this example, seven observations are required such that VGPA can overcome the local minima. In this case, both variational approximations are visually indistinguishable. 

\subsection{Code}

For more details regarding hyperparameters and implementation specifics, we refer to accompanying code available at \url{https://git.rwth-aachen.de/bcs/projects/cw/public/mbvi_sde}.

% \begin{figure}[ht]
% \centering
% \includegraphics{../../img/lorenz63_vgpa_vs_mbvi.pdf}
% \caption{Smoothing for the three-dimensional Lorenz '63 model. The three components of the system are shown from left to right. In each row, the underlying simulated state is identical but the number of observations passed to the smoothing algorithms vary. Each panel shows the ground truth (dotted black), the observations (black cross) and the posterior mean of the Gaussian process approximation (blue) and the moment-based variational approach (red). }
% \label{simg:lorenz63}
% \end{figure}

% \section{Algorithms}

\end{document}